%% file: main.tex
\documentclass{article}
\usepackage{hyperref}       
\usepackage[dvipsnames]{xcolor}
\hypersetup{
  colorlinks   = true, 
  urlcolor     =  NavyBlue, 
  linkcolor    =  Mahogany,
  citecolor   = BlueViolet 
}
\usepackage[utf8]{inputenc}
\usepackage{setspace}
\usepackage[margin=1.1in]{geometry}
\usepackage{graphicx}
\graphicspath{ {./figures/} }
\usepackage{breakcites}
\usepackage{amsmath,textcomp,amssymb,geometry,graphicx,enumerate}
\usepackage{algorithm} 
\usepackage[utf8]{inputenc} 
\usepackage[T1]{fontenc}    
\usepackage{url}            
\usepackage{booktabs}       
\usepackage{amsfonts}       
\usepackage{physics}
\usepackage{nicefrac}       
\usepackage{microtype}      
\usepackage{bm}
\usepackage{graphicx}
\usepackage{parskip}
\usepackage{algorithmic}
\usepackage{amsthm}
\usepackage{algorithm}
\usepackage{mathtools}
\usepackage{amssymb}
\usepackage{wrapfig}
\usepackage{setspace}  
\usepackage{dsfont}
\usepackage{placeins}
\usepackage{tikz}
\usetikzlibrary{fit,calc}

\input{math_commads}

\usepackage{mathtools}

\newtheorem{theorem}{Theorem}[section]

\newtheorem{lemma}[theorem]{Lemma}

\newtheorem{definition}[theorem]{Definition}
\newtheorem{assumption}[theorem]{Assumption}
\newtheorem{remark}[theorem]{Remark}

\usepackage{bigints}  

\usepackage{xpatch}
\makeatletter
\AtBeginDocument{\xpatchcmd{\@thm}{\thm@headpunct{.}}{\thm@headpunct{}}{}{}}
\makeatother

\title{Federated Fairness without Access to Sensitive Groups}

\author{%
  Afroditi Papadaki \\
  {\normalsize University College London}\\ 
  {\small \color{NavyBlue}\texttt{a.papadaki.17@ucl.ac.uk}}\\
  \and
  Natalia Martinez\\
  {\normalsize IBM Research} \\ 
   {\small \color{NavyBlue}\texttt{natalia.martinez.gil@ibm.com}}\\
  \and
  Martin Bertran\\
  {\normalsize Amazon Web Services} \\ 
   {\small \color{NavyBlue}\texttt{maberlop@amazon.com}}\\
  \and
  Guillermo Sapiro \\
  {\normalsize Duke University and Apple } \\
   {\small \color{NavyBlue}\texttt{guillermo.sapiro@duke.edu}}\\
  \and
  Miguel Rodrigues\\
  {\normalsize University College London} \\
  {\small \color{NavyBlue} \texttt{m.rodrigues@ucl.ac.uk}}}
\date{\vspace{-5ex}}

\begin{document}

\maketitle

\begin{abstract}
 Current approaches to group fairness in federated learning assume the existence of predefined and labeled sensitive groups during training. However, due to factors ranging from emerging regulations to dynamics and location-dependency of protected groups, this assumption may be unsuitable in many real-world scenarios. In this work, we propose a new approach to guarantee group fairness that does not rely on any predefined definition of sensitive groups or additional labels. Our objective allows the federation to learn a Pareto efficient global model ensuring worst-case group fairness and it enables, via a single hyper-parameter, trade-offs between fairness and utility, subject only to a group size constraint. This implies that any sufficiently large subset of the population is guaranteed to receive at least a minimum level of utility performance from the model. 
The proposed objective encompasses existing approaches as special cases, such as empirical risk minimization and subgroup robustness objectives from centralized machine learning. We provide an algorithm to solve this problem in federation that enjoys convergence and excess risk guarantees. Our empirical results indicate that the proposed approach can effectively improve the worst-performing group that may be present
without unnecessarily hurting the average performance, exhibits superior or comparable performance to relevant baselines, and achieves a large set of solutions with different fairness-utility trade-offs.
\end{abstract}

\section{Introduction}
\input{Section1} 

\section{Related Work}\label{sec:related_work}
\input{Section2}

\section{Problem Formulation}\label{sec:prob_form}
\input{Section3}

\section{Optimization Method}\label{sec:optimization}
\input{Section5}

\section{Algorithmic Analysis}\label{sec:analysis}
\input{Section6}

\section{Experimental Results}\label{sec:experiments}
\input{Section7}

\section{Conclusions and Limitations}\label{sec:conclusion}
\input{Section8}

\section*{Acknowledgments}
UCL authors were supported by Cisco under grant \#217462. GS is partially supported by NSF, ONR, NGA, and the Simons Foundation.
\newpage

\bibliographystyle{apalike}
\bibliography{main.bib}

\newpage
\counterwithin{table}{section}
\appendix

\input{appendix/appendix_main}

\end{document}

%% file: math_commads.tex
\usepackage{amsmath,amsfonts,bm}

\def\eqref#1{equation~\ref{#1}}

\def\1{\bm{1}}

\DeclareMathAlphabet{\mathsfit}{\encodingdefault}{\sfdefault}{m}{sl}
\SetMathAlphabet{\mathsfit}{bold}{\encodingdefault}{\sfdefault}{bx}{n}

\newcommand{\E}{\mathbb{E}}

\newcommand{\R}{\mathbb{R}}



%% file: Section1.tex
Federated learning (FL) is a paradigm that allows multiple entities/clients, usually coordinated by a central server, to collaboratively train a model to achieve some common learning objective on their combined data. The clients do not share their raw data, but rather limited information (e.g., model updates, risk values, etc.) during the training procedure \cite{DBLP:journals/corr/KonecnyMRR16,DBLP:journals/corr/KonecnyMYRSB16}. 
Such a learning paradigm has been widely used for high-stakes decision-making applications such as open banking \cite{Long2020} and genomics research \cite{weinstein2013cancer} to guarantee that data is kept decentralized and private.

A key concern in federated learning is ensuring the fairness of the resulting model across various sensitive groups \cite{DBLP:journals/corr/KonecnyMRR16}, where these groups may be present in different proportions across clients.
This challenge has been the focus of many works -- such as \cite{BGL,papadaki} -- where it is assumed that clients are aware of the sensitive groups during training and can accurately assign group memberships to each data point. Nevertheless, knowledge of the sensitive groups and access to group memberships might not always be feasible for various reasons. For instance, in a scenario where various hospitals collaborate to learn a group-fair diagnostic model through federated learning, it is often unrealistic to expect that sensitive groups are identified, or that medical records include accurate information about patients' race, religion, or sexual orientation. This is because obtaining these labels can be costly \cite{10.1145/3025453.3025930}, require specialized knowledge, or breach privacy regulations (e.g., GDPR \cite{EUdataregulations2018} or CCPA \cite{ccpa}) that restrict the collection and utilization of certain types of personal information.

\begin{wrapfigure}{r}{0.55\columnwidth}
  \vspace{-1.55\baselineskip}\includegraphics[width=0.55\columnwidth]{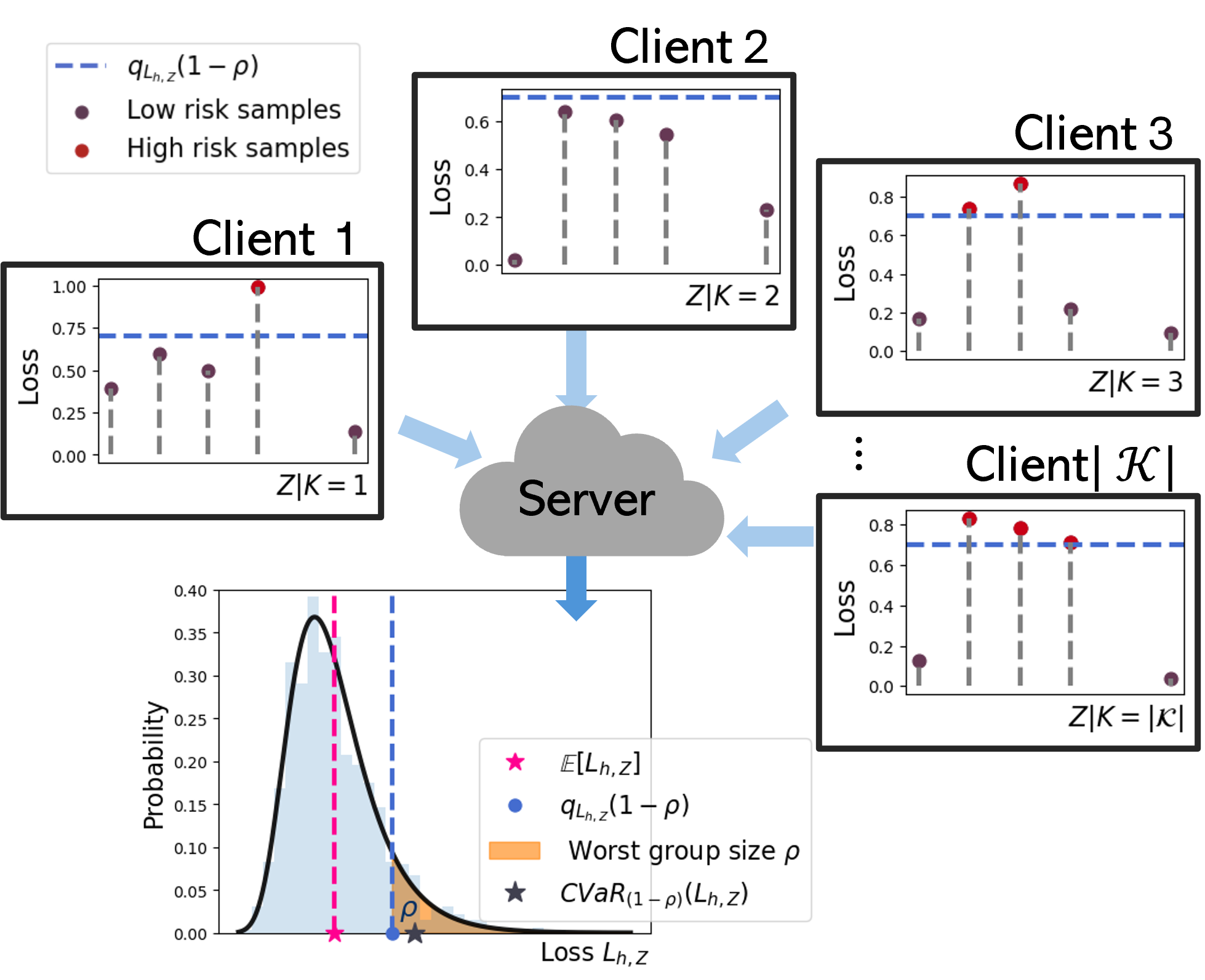}\vspace{-1.\baselineskip}

    \caption{We consider a federated learning setting with  $K \in \mathcal{K}$ clients. Our method maximizes the performance of the worst possible group of size $\rho$ that can be formulated from the union of the local individuals/samples $z$, that is, within a model class, no other model performs better on its worst $\rho$ fraction of the samples. Equivalently, no other model has a lower $(1-\rho)$-th superquantile of its loss distribution. We achieve this objective at the lowest possible cost to the non-critical samples. We make no assumptions on the distribution of the worst-performing samples amongst clients, and note that (a) worst-performing sensitive groups might not align with a single conventional demographic, and (b) `groups' and `clients' are not synonyms in our setting.
    }\label{fig:proposed_fw}    
\end{wrapfigure}

In this work, we address the challenge of achieving federated group fairness for \textit{any} potential definition of sensitive groups, even those defined \textit{after} the model is deployed, assuming the group definitions cover a significant, but configurable, subset of the overall population. 
We focus on the Rawlsian max-min notion of fairness \cite{rawls2001justice} and on the no-harm principle where degrading the performance of a particular group can only be justified if it improves a disadvantaged group. 
We introduce a new learning objective, the Relaxed Conditional Value-at-Risk (RCVaR), designed to enhance the performance of the worst-off subset of data samples without unnecessarily reducing the performance of the remaining ones. 
Our objective depends on two parameters: (a) the trade-off parameter $\epsilon$ that allows to flexibly define the importance added to the average utility versus (minimax) fairness; and (b) the constraint $\rho$ which bounds the size of the worst-case group, depending on some common policy/preference. 
Our approach enables clients to (a) identify any global and potentially critical sensitive group, independently of whether it exists on its local distribution during training; and (b) learn a global hypothesis that allows a trade-off between mean performance and fairness. An example of our method is provided in Figure \ref{fig:proposed_fw}. 

\paragraph{Main Contributions.} 
    To the best of our knowledge, we are the first to address the challenge of (minimax) Pareto federated group fairness with inhomogeneous and unknown sensitive groups, where clients and groups are not aligned. We introduce a new fairness-aware objective -- RCVaR -- that allows improving the performance of the high-risk samples, subject to only a group size constraint, while ensuring the best possible performance on the remaining samples.
    We draw formal connections between the proposed objective and {existing ones, such as DRO \cite{DRO} and BPF \cite{pmlr-v139-martinez21a}}, demonstrating that RCVaR can also be used for learning Pareto subgroup robust models in centralized settings. 
    We then introduce an algorithm -- FedSRCVaR -- that solves a smoothed approximation of RCVaR in federated learning settings. We establish the algorithm's convergence and excess risk properties, and show that the proposed objective can be easily federalized compared to centralized learning objectives. Finally, we empirically study the wide range of solutions that can be achieved by our approach through the trade-off parameter for various group sizes, and compare our method against other relevant baselines in centralized and federated learning settings using real datasets. 

%% file: Section2.tex
\paragraph{Minimax Fairness in ML.} 
Minimax fairness criterion -- or Rawlsian max-min fairness from a utility point of view \cite{rawls2001justice,martinez2020,diana2020convergent}--, is the no-harm (Pareto optimal) approach to equality of errors \cite{diana2020convergent} since it aims to improve the model's utility for the worst-performing group without unjustifiably diminishing the performance of other groups. 
In the case of unknown sensitive groups, fairness is measured as the utility perceived by the worst-served subset of individuals / samples \cite{DRO}, not as a difference in performance or outcomes between groups. 
In this work, we leverage this notion to learn a model that improves the worst-performing group in the most general formulation possible, that is, the worst group (subset) of samples distributed across all the clients in the federation. 
We note that other fairness definitions such as statistical parity, equality of odds, or equality of risks can conflict with each other, leading to sub-optimal outcomes where certain groups are harmed without any improvement to other groups, as discussed in \cite{kleinberg,chen2018classifier,article_mmpf_}.

\paragraph{Fairness without Sensitive Groups in Centralized ML.} A recent body of literature on centralized machine learning deals with group fairness without explicit protected groups. Early works \cite{proxy-fairness,10.1287/mnsc.2016.2579} address the problem of unknown group annotations by designing a proxy variable that replaces the true sensitive group variable so that conventional group fairness methods can be deployed. These approaches require knowledge of the true sensitive groups, though the sample's group labels are considered unavailable, which is hard to obtain for many applications.

Subsequent research addresses fairness without group labels through (sub)group robustness. DRO \cite{DRO} optimizes the performance of samples that exceed a specified risk threshold. Similarly, blind Pareto fairness (BPF) \cite{pmlr-v139-martinez21a} minimizes the worst-case risk across any possible group distribution formulated by the training data, subject to a group size constraint, while ensuring that the model is Pareto efficient \cite{miettinen2012nonlinear}. \cite{ARL} and \cite{sohoni2020no} use auxiliary models to discover the worst-performing group. 

These works are designed for centralized settings where data is collected and processed by a single entity. Our research focuses on federated learning settings where the data is heterogeneously distributed across multiple clients and cannot be shared. 
We build upon DRO and BPF, due to the interpretability of their adversary, to propose a relaxed superquantile criterion that allows achieving different levels of (minimax) group fairness through a hyperparameter $\epsilon$ for a fixed group size $\rho$. Both DRO \cite{DRO} and BPF \cite{pmlr-v139-martinez21a} can be seen as a particular case of our proposed RCVaR formulation, as discussed in Section \ref{sec:prob_form}. Similar to BPF and differing from DRO, RCVaR considers properly Pareto optimal solutions and incorporates trade-offs between average performance and fairness. 

\paragraph{Fair Federated Learning.} {The federated learning literature explores various notions of fairness, with a significant portion of these works \cite{AFL,DRFA,fedmgda+,DBLP:journals/corr/abs-2108-0274} focusing on achieving fairness across clients. This is typically done by optimizing the model to enhance the performance of the client (or cluster of clients) that exhibits the lowest performance. However, as \cite{papadaki} formally demonstrates, fairness across clients does not guarantee fairness across different sensitive populations within those clients, except under specific circumstances, such as when each participant's dataset exclusively represents a single sensitive group. 

The works in \cite{cui2021addressing,zhang2021unified} propose approaches to achieve group fairness within each individual client (hence, targeting only groups available within the client at training and testing time), while the works in \cite{enforcing-fairness,papadaki,BGL} propose methods for learning models that are fair across all the known sensitive groups available in the clients, even if some participants have access to a subset of them during training.} To our knowledge, the only work that considers scenarios, where group data cannot be leveraged, is \cite{juarez2023you}, but it assumes that the collection of sensitive groups is known apriori (though not used due to privacy concerns). Hence, they recommend local differential private mechanisms to alleviate privacy issues, while allowing information about the group memberships to be used for learning fair models. 

Similar to the aforementioned approaches, our goal is to learn a model that ensures group fairness across any sensitive groups that exist in the clients' data. However, we distinguish ourselves by focusing on a more complex scenario, where the sensitive populations are defined in terms of the performance of a given model, and cannot be labeled before the model learning process takes place. Hence, no a-priori group information can be incorporated into the training phase.

%% file: Section3.tex
 
\subsection{Minimax Fairness for Worst  Case Scenarios}
\label{preliminaries}
Let the pair of random variables {$Z=(X,Y) \in \mathcal{X} \times \mathcal{Y}$} represent the input features and categorical targets, generated from a distribution ${p(Z)=p(X,Y)}$.  
Let also $\ell:\Delta^{|\mathcal{Y}|-1} \times \Delta^{|\mathcal{Y}|-1} \to {\R}_+$ be a loss function and $h$ be a hypothesis drawn from the hypothesis class  $\mathcal{H}=\{h:\mathcal{X} \rightarrow \Delta^{|\mathcal{Y}|-1} \}$, where $\Delta^{|\mathcal{Y}|-1}$ is the probability simplex over $\mathcal{Y}$.  
We assume there is no prior knowledge about the groups or sensitive labels associated with any $z$.

In particular, let $L_{h,Z} \coloneqq \ell(h;Z)$ denote a random variable representing the loss associated with a hypothesis $h \in \mathcal{H}$. For a predefined probability $\rho \in (0, 1)$, the $(1-\rho)$-quantile function is defined as 
\begin{equation}
    q_{{L_{h,Z}}}(1-\rho):=\inf \big \{ \beta \in \R:p({L_{h,Z}} \leq \beta) \geq 1-\rho \big \},
\end{equation} 
and the $(1-\rho)$-superquantile, also known as the conditional value-at-risk (CVaR), function at confidence level $(1-\rho)$ is defined as 
\begin{equation}\label{cvar_central}
     CVaR_{(1-\rho)}({L_{h,Z}})=\mathop{\E}\limits_{ {Z }}[{L_{h,Z}}|{L_{h,Z}}\geq   q_{{L_{h,Z}}}(1-\rho)].
\end{equation}
The quantity in Eq. \ref{cvar_central} is a measure of the upper tail behaviour of the distribution $p({L_{h,Z}})$ and, as shown in  \cite{ROCKAFELLAR2014140}, it can be expressed for a bounded loss function, i.e., $0\leq\ell(h;z)\leq B$, $B>0$, $\forall z \in \mathcal{Z}$, as 
\vspace{-0.02in}
\begin{equation}\label{cvar_global_variational}
   CVaR_{(1-\rho)}({L_{h,Z}}) = \min\limits_{ c\in [0,B]}  
     c+ \frac{1}{\rho}\mathop{\E}\limits_{ {Z\sim p(Z)}}[  ({L_{h,Z}} -c)_+ ],
\end{equation}
where $(\cdot)_+:=\max\{0, \cdot \}$ and the second term represents the regret of any positive realizations of ${L_{h,Z}}$. Note that the argument that minimizes the objective in Eq. \ref{cvar_global_variational}
is the quantile $q_{{L_{h,Z}}}(1-\rho)$.  

Therefore, we can formulate the problem of learning a minimax group fair hypothesis with no knowledge of sensitive group populations as follows,
\begin{equation}\label{minh_cvar_global_variational}
     h^*,c^*=\arg\min\limits_{h \in \mathcal{H},c\in [0,B]}  c+ \frac{1}{\rho}\mathop{\E}\limits_{ {Z\sim p(Z)}}[  ({L_{h,Z}} -c)_+ ]. 
\end{equation}
The optimization problem in Eq. \ref{minh_cvar_global_variational} allows for minimax fair solutions, since it optimizes for the worst tail risk with sample size $\rho$, or equivalently the worst performing samples that exceed threshold $c$. 

Nevertheless, Eq. \ref{minh_cvar_global_variational} ignores any data that is not considered high-risk (i.e., samples that are below the threshold $c$) and hence allows for solutions that are weakly Pareto optimal \cite{miettinen2012nonlinear}. The formal definition of weakly Pareto optimality is provided in Definition \ref{def:weakly_pareto}.

\begin{definition}[Weak Pareto optimality]\label{def:weakly_pareto}
    A hypothesis $h^* \in \mathcal{H}$ is weakly Pareto optimal if for any possible sensitive group $g \in \mathcal{G}$, $\nexists  h \in \mathcal{H}: \mathop{\E}\limits_{Z|g}[\ell(h;Z)]< \mathop{\E}\limits_{Z|g}[\ell(h^*;Z)]$ . 
\end{definition}
Weakly Pareto optimal hypotheses can potentially compromise the performance of low-risk sensitive groups, especially when the input space exhibits regions of no uncertainty regarding the target class (i.e., the data is perfectly separable). 

We next propose an objective that does not unnecessarily harm the low-risk sensitive group. We also focus on the more challenging scenario of federated learning where the data might be heterogeneously distributed across clients and the goal is to achieve a solution equivalent to centralized machine learning.

\subsection{Federated Minimax Blind Fairness}
In the context of federated learning, we consider an additional random variable $K$ representing the clients in the federation. Each client $k \in \mathcal{K}$ holds data modelled by its own local distribution ${p(Z|K=k)}=p(X|K=k)p(Y|X,K=k)$. {Therefore, the data of the entire federation can be described via the mixture distribution} ${p(Z)}=\sum\limits_{k \in \mathcal{K}}p(K=k) {p(Z|K=K)}$.

Let ${L_{h,Z|K=k}:= \ell(h;Z)}$, with $ {Z\sim p(Z|K=K)}$, denote a random variable representing the local loss on $Z$ induced by a hypothesis $h$ in client $k$. We formulate relaxed conditional value-at-risk (RCVaR), a generalization of the objective in Eq. \ref{cvar_global_variational} that (a) produces properly Pareto optimal solutions, and (b) is suitable for federated learning scenarios with inhomogeneously distributed data across clients, as follows, 
\begin{equation}\label{Fed_cvar_tradeoff}
\begin{array}{cc}
      \min\limits_{h \in \mathcal{H}}  \Big \{(1-\epsilon) CVaR_{(1-\rho)}({L_{h,Z}}) + \epsilon  \mathop{\E}\limits_{Z \sim {{p(Z)}}}[{L_{h,Z}}]\Big \}\\  \\
     = \min\limits_{h \in \mathcal{H},c \in [0,B]} \mathop{\E}\limits_{K } \Big [ \mathop{\E}\limits_{Z|k} \big[(1-\epsilon)
    \bigg( c+     \frac{1}{\rho}(L_{h,Z|K=k }-c)_+ \bigg)   +  \epsilon  {L_{h,Z|K=k }}\big] \Big],
\end{array}
\end{equation}
 where a hyperparameter $\epsilon \in [0,1]$ induces a trade-off between the average and worst-case group performances.
The threshold $c$ is uniformly applied across all clients in the federation to identify the samples that belong to the \textit{global} high-risk and the low-risk groups.\footnote{We focus on the high- and low-risk groups within the data distribution, accommodating any potential protected group, either binary or multigroup. The equivalence of minimax worst-case group performance on a two-group or $n$-group formulation, is shown in Lemma 3.1 in \cite{pmlr-v139-martinez21a}.} Also, it allows the selection of any partition of overall size $\rho$ across clients, and to consider larger local group sizes in clients with worse performances from clients with high performance. Otherwise, assigning the same $\rho$ across clients would yield a model that minimizes the server average of the worst per-client partition, which yields an overall partition of size $\rho$, but is a weaker adversary than our framework and does not guarantee the same performance as centralized settings.

Moreover, RCVaR ensures minimax properly Pareto fairness, where the worst possible group is formed by the high-risk samples subject to a predefined group size constraint $\rho$. This combination of minimax fairness and proper Pareto optimality is crucial for high-stakes decision-making to ensure that the produced model does not unnecessarily harm well-performing groups. The formal definition of proper Pareto optimality is offered in Definition \ref{def:properly_pareto}. 
\begin{definition}[Proper Pareto optimality]\label{def:properly_pareto}
    A hypothesis $h^* \in \mathcal{H}$ is properly Pareto optimal if for any possible sensitive group $g \in \mathcal{G}$,   $\nexists h \in \mathcal{H}: 
    \mathop{\E}\limits_{Z|g}[\ell(h;Z)]\leq \mathop{\E}\limits_{Z|g}[\ell(h^*;Z)]$, and $\exists g' \in \mathcal{G}: \mathop{\E}\limits_{Z|g'}[\ell(h;Z)]\break<\mathop{\E}\limits_{Z|g'}[\ell(h^*;Z)]$.  
\end{definition} 
By nature, and similar to DRO and BPF, RCVaR guarantees that no group partition that encompasses more than $\rho$ of the total population, pre-defined or not, will experience a worse performance than the one obtained by optimizing Eq. \ref{Fed_cvar_tradeoff}.

 \paragraph{Connections to CVaR and DRO.} For $\epsilon=0$, RCVaR is equivalent to CVaR in centralized learning settings, thereby allowing for minimax weakly Pareto optimal solutions. Also, CVaR is the dual formulation of DRO \cite{DRO} for specific uncertainty sets, as formally shown in Proposition 3 in \cite{DRO} and Lemma 2.1 in \cite{duchi}. Thus, RCVaR with  $\epsilon=0$ is the federated formulation of DRO \cite{DRO}. In contrast to standard CVaR and DRO, the additional utility term and trade-off parameter $\epsilon$ in Eq. \ref{Fed_cvar_tradeoff} enables a larger set of achievable solutions that are proper Pareto optimal, as we discuss in Section \ref{sec:experiments}.
There are also algorithmic differences between optimizing DRO and RCVaR that we describe in Section \ref{sec:optimization}.

\paragraph{Connections to BPF.} Setting $\epsilon \approx 0$, sufficiently small but non-zero, produces a hypothesis that is minimax (properly) Pareto fair, since it focuses on the worst-performing samples, while still utilizing the remaining samples with a small priority $\epsilon$. 
We argue that for such $\epsilon$ value, the LHS of Eq. \ref{Fed_cvar_tradeoff} is a new expression for Pareto subgroup robustness suitable for centralized learning settings as well. We detail how RCVaR relates to BPF in Appendix \ref{appdx:BPF_connections}. 
Due to the connection of our objective with BPF, we argue that our objective inherits BPF's properties presented in \cite{pmlr-v139-martinez21a}, including the fact that there exists a critical partition size $\rho$ that leads to the uniform classifier for sufficiently small $\epsilon$ values.

\paragraph{Connections to ERM and FedAvg.} If $\epsilon = 1$, our objective reduces to the vanilla-ERM objective in \cite{fedavg}. For any other intermediate value of $\epsilon \in (0,1)$, we obtain a trade-off between utility and subgroup robustness. 
{To better understand the set of trade-offs achieved by RCVaR, we offer an illustrative example in Figure \ref{fig:trade_offs_xaxis_rho}. 
We emphasize that the value of $\epsilon$ is predefined and fixed, and therefore, we leave it to the policy maker(s) to determine the fairness-utility compromise. An additional advantage of the objective in Eq. \ref{Fed_cvar_tradeoff} is that it can be easily federalized, as shown in the sequel.

%% file: Section5.tex
In real applications, each client holds only a finite dataset {$D_k=\{ z^k_i\}_{i=1}^{n_k}$, 
with $z^k_i=(x^k_i,y^k_i)$,} sampled from the true distribution ${p(Z|K=k)}$, with $D=\bigcup\limits_{k \in \mathcal{K}}D_k$ being the dataset containing all the data samples available across clients of size $n=\sum\limits_{k \in \mathcal{K}}n_k$. 
{Hence, in the sequel we use the empirical form of RCVaR given by
\begin{equation}\label{empirical_Fed_cvar_tradeoff}
    \min\limits_{\bm \theta \in \Theta, c \in [0, B]} \sum\limits_{k \in \mathcal{K}}\frac{1}{n}\sum\limits_{i =1}^{n_k} f(\bm \theta,c;z^k_i),
\end{equation}
where 
\begin{equation}\label{auxiliary_function}
    f(\bm \theta,c;z)= (1-\epsilon)[c + \frac{1}{\rho}(\ell(\bm \theta;z) - c)_+] + {\epsilon} \ell(\bm \theta;z),
\end{equation}
and $\bm \theta \in \Theta$ is a vector that parametrizes the hypothesis $h \in \mathcal{H}$, and correspondingly change to the notation $\ell(\theta;z)$ instead of $\ell(h;z)$. We next offer a federated learning algorithm to solve Eq. \ref{empirical_Fed_cvar_tradeoff} that relies on a smoothed version $\Tilde{f}(\cdot)$ of the non-smooth function $f(\cdot)$.
 
 In our federated learning setting (a) every client uses a batch size $b_k\leq n_k$ of data samples at each training iteration, (b) each client might use each local data sample more than once during the training, and (c) there are $T$ communications between the clients and server. 
This realistic setting makes our algorithmic design and analysis challenging since -- in order to develop a simple algorithm with strong theoretical guarantees -- we need an objective that is continuously differentiable for all $z$. 
Unfortunately, even for smooth loss functions $\ell$, the $f$ in our current objective is non-smooth due to the plus function $(\cdot)_+$. To overcome this issue we consider a proxy problem of the RCVaR in Eq. \ref{empirical_Fed_cvar_tradeoff}, which relies on a smooth approximation.

\subsection{Smooth Approximation of RCVaR }

We consider the family of smoothed plus functions that satisfy  Definition \ref{def:smoothed_plus_function}.

\begin{definition}[Smooth Approximation \cite{smooth_plus_func}]\label{def:smoothed_plus_function} For a smoothing parameter $\gamma \in \R_+$ and for any $m \in \R$, a $(\frac{2}{\gamma})-$smooth convex 
function $s:\R\to\R_+$ approximates a plus function $(\cdot)_+$, if it satisfies  
$0\leq s(\cdot)-(\cdot)_+\leq  {\gamma}$.
\end{definition}

A smooth plus function $s(\cdot)$ becomes a more accurate approximation of the plus function, for small values of $\gamma$, as discussed in Section \ref{sec:analysis}.
The designed algorithm and its analysis support any function that is consistent with Definition \ref{def:smoothed_plus_function} (e.g., soft ReLU \cite{Peng1999}, Zang smooth plus function \cite{zang}, piecewise quadratic smoothed plus function \cite{quadratic}), rather than a specific smoothed plus function. The empirical smooth approximation of RCVaR is formulated as 
\begin{align}\label{smooth_empirical_Fed_cvar_tradeoff}
\begin{array}{cc}
    \min\limits_{\bm \theta \in \Theta, c \in [0, B]} \frac{1}{n} \sum\limits_{k \in \mathcal{K}}\sum\limits_{i =1}^{n_k} \Tilde{f}(\bm \theta,c;z^k_i)
\end{array}
\end{align}
where
\begin{equation*}
    \Tilde{f}(\bm \theta,c;z)=(1-\epsilon)[c + \frac{1}{\rho}s(\ell(\bm \theta;z^k_i) - c)] + {\epsilon} \ell(\bm \theta;z^k_i).
\end{equation*}  
 
\subsection{FedSRCVaR: Federated Smoothed RCVaR Algorithm}

Next, we introduce a federated learning algorithm designed to address Eq. \ref{smooth_empirical_Fed_cvar_tradeoff}, namely FedSRCVaR. The algorithm is outlined in Algorithm \ref{alg:fedcvar}. 

Our algorithm performs the following successive steps for $T$ communication rounds: \textbf{(a)} the clients receive the global model-threshold pair $(\bm \theta^t,c^t)$ of the current round from the server; \textbf{(b)} The clients perform $\tau$ local updates on the model parameters and the threshold using $b_k-$samples; \textbf{(c)} The clients return the updated pair $(\bm \theta_k^{t+1},c_k^{t +1})$ to the server; \textbf{(d)} Finally, the server produces the new model-threshold pair $\bm (\bm\theta^{t+1},c^{t+1})$ by averaging the received client updates. 

We denote $\textit{proj}_{[0,B]}$ the metric projection operator onto the set $[0,B]$ to ensure that the threshold $c$ has a valid value within the specified range. The server averages the clients updates using the relative weights $\frac{b_k}{\sum_{k \in \mathcal{K}}b_k}$, with $b_k$ being the batch size of client $k$ which is proportional to $n_k$ to allow clients to use a fraction of their local dataset, accommodating constraints such as computation limitations on the client side.
The algorithm outputs the average model-threshold pair over the total communications $( \overline{\bm \theta}_{T},\overline{c}_{T})$ that is produced after $|\mathcal{K}|\tau T$ total updates. 
\input{algorithms/algorithm_merged}

\paragraph{Comparison to BPF and DRO methods:} All methods use parameter $\rho$ in their design. In addition to the flexibility of FedSRCVaR offered by the $\epsilon$ parameter discussed in Section \ref{sec:prob_form}, a key distinction between FedSRCVaR and DRO \cite{DRO} lies in their threshold learning approaches. Our algorithm employs (distributed) projected gradient descent with periodic averaging, while DRO relies on a (centralized) binary search method. 
Furthermore, FedSRCVaR can easily be deployed in dynamic learning settings, such as online learning settings, where the global model is trained using a continuous stream of new data arriving sequentially in real-time. BPF requires estimating and optimising per-sample adversarial weights at each optimization round, managing and accessing the last risk evaluation for every sample and adjusting the set from which adversarial weights are selected become computationally expensive.
Finally, FedSRCVaR is lightweight even for $\tau=1$, since it requires only the exchange of the updated model-threshold pair between clients and the server, which makes its communication overhead insignificant compared to the communication costs and additional privacy concerns that are required for the federalization of BPF. 
We share more information about this comparison in Appendix \ref{appdx:BPF_federalization}.

%% file: algorithms/algorithm_merged.tex
\begin{algorithm}
    \caption{FedSRCVaR Algorithm}
    \label{alg:fedcvar}
    { \textbf{Inputs: }} $\mathcal{K}$: set of clients, $T$: communication rounds, $\tau$: local rounds, $\eta$: learning rate for model $\theta$ and quantile $c$, $\epsilon \in (0,1]$: trade-off parameter, $\rho \in (0,1)$: parameter for probability-level, $b_{k} $: local batch size.
      
    \begin{algorithmic}[1]
     
    \STATE Server initializes $\bm \theta^1$ randomly and   sets $c^1=B=1$. 

    \FOR{$t=1$ to $T$}
    
    \STATE Server \textbf{broadcasts} model-threshold $(\bm \theta^{t},c^{t})$

    \FOR{each client $k \in \mathcal{K}$ \textbf{ in parallel}}
    
     \STATE Randomly sample a data batch of size $b_{k} $ 
    \FOR{$j=1$ to $\tau$}
     
    \STATE Set $({\bm \theta}^{t,j=1},c^{t,j=1})=({\bm \theta}^{t},c^{t})$
    
     { \STATE  $\bm \theta^{t,j+1}_k \leftarrow \bm \theta^{t,j} - \eta  \nabla_{\theta}
     \bigg \{ 
     \sum\limits_{i=1}^{b_{k} } 
     \frac{\Tilde{f}(\bm \theta^{t,j},c^{t,j};z_i^k) }{b_{k} }
     \bigg \}
     $ }, 
     {  $c^{t,j+1}_k \leftarrow c^{t,j } - \eta  \nabla_{c } 
     \bigg 
     \{ 
     \sum\limits_{i=1}^{b_{k} } \frac{ \Tilde{f}(\bm \theta^{t,j},c^{t,j};z_i^k)}{b_{k} }  \bigg 
     \}    
     $ }
       
    \STATE Return local pair $(\bm \theta^{t,\tau}_k ,c^{t,\tau}_k )$ 
    \ENDFOR 
    \ENDFOR         
     {  \STATE Server computes  ${\bm \theta}^{t+1} \leftarrow \sum\limits_{k \in \mathcal{K}} \frac{b_k{{\bm \theta}^{t,\tau}_k}}{\sum\limits_{k \in \mathcal{K}}b_k}  $, ${c} ^{t+1}  \leftarrow \prod\limits_{c \in [0, B]} \bigg (\sum\limits_{k \in \mathcal{K}} \frac{b_k{c^{t,\tau}_k}}{\sum\limits_{k \in \mathcal{K}}b_k}\bigg )  $    }
    
     \ENDFOR
    
    \end{algorithmic}
    \textbf{Outputs: } {$ \overline{\bm\theta}_T=\frac{1}{T}\sum\limits_{t \in [T]}\bm \theta^t $  and  $ \overline{c}_T=\frac{1}{T}\sum\limits_{t \in [T]}c^t $}  
\end{algorithm}

%% file: Section6.tex
We now examine the performance of Algorithm \ref{alg:fedcvar} by assessing the associated convergence rate and expected excess risk. 
Our analysis relies on the following assumptions.

\begin{assumption}\label{ass:loss_function}
The loss function $\ell(\bm \theta,z)$ is convex wrt $z$, $G-$Lipschitz, and $\beta-$smooth function of range $[0,B]$, with $B=1$, for all $z$ and $\bm \theta \in \Theta$.
\end{assumption} 
\begin{assumption}\label{ass:convex_sets}
The set $\Theta \subseteq \R^d$ is convex with $||\bm  \theta- \bm \theta'||\leq M$, for any $ \bm \theta, \bm \theta' \in \Theta$.
\end{assumption}

\begin{assumption}\label{ass:bounded_variance} 
For any model-threshold pair $({\bm \theta},c) \in \Theta\times[0,B]$, each client $k \in \mathcal{K}$ can query an unbiased stochastic gradient, i.e., $\E\big[\nabla \{\sum\limits_{i =1}^{b} \frac{1}{b}\Tilde{f}(\bm {\bm \theta},c;z^k_i)\}\big]=\nabla  \mathop{\E}\limits_{ Z|k }\big[\Tilde{f}(\bm {\bm \theta},c;Z)\big] $, with $\sigma^2$ -- uniformly bounded variance, i.e.,
\begin{equation*}
\begin{array}{c}
\E \bigg[ \bigg|\bigg| \frac{1}{b} \nabla \{\sum\limits_{i =1}^{b }\Tilde{f}(\bm {\bm \theta},c;z^k_i)\}  - \nabla  \mathop{\E}\limits_{ Z|k }\big[\Tilde{f}(\bm {\bm \theta},c;Z)\big]  \bigg|\bigg|^2 \bigg]\leq \sigma^2.
\end{array}
\end{equation*}
\end{assumption}
\begin{assumption}\label{ass:bounded_local_global_grads} 
The difference between local and global gradients is $\mu-$uniformly bounded, meaning that 
\begin{equation}
\begin{array}{c}
    \max_{k \in \mathcal{K}}\sup\limits_{({\bm \theta},c) \in \Theta \times [0,B]}  \bigg|\bigg| \nabla \mathop{\E}\limits_{ Z|k }\big[\Tilde{f}(\bm {\bm \theta},c;Z)\big] - \nabla \mathop{\E}\limits_{ Z }\big[\Tilde{f}(\bm {\bm \theta},c;Z)\big]  \bigg|\bigg| \leq \mu.
\end{array}
\end{equation}
\end{assumption}

The definitions of these properties are provided in Appendix \ref{appdx:definitions}. 
Under these assumptions, we can establish the core properties of the smooth and non-smooth functions, $f$ and $\Tilde{f}$, required for our analysis, in Lemma \ref{lemma:properties}.
The proof is provided in Appendix
\ref{appdx:smooth_approx}. 

\begin{lemma}\label{lemma:properties} Let Assumption \ref{ass:loss_function} hold. Let also $s:\R\to\R_+$ be a $\frac{2}{\gamma}-$smooth convex function. Then,
\begin{enumerate}
    \item The functions $f $ and $\Tilde{f} $ are convex for every $z$.
    \item The function $f $ and the smoothed function $\Tilde{f} $ are $G_{\rho,\epsilon}-$ Lipschitz for all $z$ with 
    \begin{equation*}
    \begin{array}{c}
        G_{\rho,\epsilon} =\max \bigg\{\frac{1}{\rho}\sqrt{{G^2(1-\epsilon +\epsilon \rho)^2 + (1-\epsilon)^2({\rho}-1)^2 }},\sqrt{G^2  \epsilon^2 + (1-\epsilon)^2} \bigg\}.
        \end{array}
    \end{equation*}
    \item The function $\Tilde{f}$ is $ \big(\frac{(1-\epsilon)}{\rho}(\beta + \frac{2}{\gamma} G^2) + \epsilon \beta  \big) -$smooth.
    \item For any model $\bm \theta \in \Theta$ we have that 
    \begin{equation}\label{ineq_functions}
    \begin{array}{c}
        f(\bm \theta,c;z)\leq \Tilde{f}(\bm \theta,c;z)\leq f(\bm \theta,c;z) + \frac{(1-\epsilon)}{\rho}{\gamma}.
        \end{array}
    \end{equation}
\end{enumerate}
\end{lemma} 
\vspace{-.5\baselineskip}
We remark that Eq. \ref{ineq_functions} bounds the smooth function $\Tilde{f}$ using $f$, which allows us to express our guarantees in terms of $f$, but importantly prove them in terms of the more analytically tractable $\Tilde{f}$.

\subsection{Convergence of Algorithm \ref{alg:fedcvar} }\label{appdx:multiround_case}
We begin by characterizing the optimization error given by
 \begin{equation*}
 \begin{array}{c}
 \mathcal{E}_{opt}=  
\mathop{\E}\limits_{\mathcal{A},D}\bigg[\sum\limits_{z \in D}  \frac{{f}(\overline{\bm  \theta}_T,\overline{c}_T;z )}{n} \bigg] 
 -\mathop{\E}\limits_{D}\bigg[\sum\limits_{z \in D} \frac{{f}( \bm \theta^*_D,c^*_D;z)}{n} \bigg],
 \end{array}
\end{equation*}
where $(\overline{\bm  \theta}_T,\overline{c}_T)$ is the average model-threshold pair after $T$ rounds of Algorithm \ref{alg:fedcvar}, $( \bm \theta^*_D,c^*_D)$ is the model-threshold pair that minimizes the smoothed objective in Eq. \ref{smooth_empirical_Fed_cvar_tradeoff} and the outer expectation in the first term is taken over the randomness induced by our randomized algorithm $\mathcal{A}$ and the samples $D$, and in the second term with respect to the dataset $D$. {This error captures how well the produced pair $(\overline{ \bm \theta}_T,\overline{c}_T)$ approximates the optimal empirical pair $( \bm \theta^*_D,c^*_D)$ in terms of the true (non-smooth) objective function.}

The next lemma offers a bound to the optimization error $\mathcal{E}_{opt}$. The proof -- detailed in Appendix \ref{appdx:convergence} -- leverages results for local-update gradient-based algorithms presented in Theorem 1 in \cite{advancesFL2}.

\begin{lemma}[Convergence of FedSRCVaR]\label{lemma:multi_fedcvar_convergence}
Let the assumptions \ref{ass:loss_function} -  
\ref{ass:bounded_local_global_grads} hold, $(\bm \theta_D^\star,c_D^\star)$ be the minimizer of Eq. \ref{smooth_empirical_Fed_cvar_tradeoff}, 
and a learning rate  \vspace{-0.5\baselineskip}
\begin{multline}\label{eq:lr_condition}
        \eta=   \min \bigg\{\frac{\sqrt{|\mathcal{K}|}\sqrt{M^2+1^2}}{\sigma\sqrt{\tau T} },  \bigg (\frac{M^2+1^2}{\sigma^2  \tau^2 \big(\frac{(1-\epsilon)}{\rho}(\beta + \frac{2}{\gamma} G^2) + \epsilon \beta  \big)  T}\bigg )^{\frac{1}{3}},
         \\
        \frac{1}{4 \big(\frac{(1-\epsilon)}{\rho}(\beta + \frac{2}{\gamma} G^2) + \epsilon \beta  \big)  },  \frac{(M^2+1^2)^{\frac{1}{3}}}{\tau(\mu^2 \big(\frac{(1-\epsilon)}{\rho}(\beta + \frac{2}{\gamma} G^2) + \epsilon \beta  \big)  T)^{\frac{1}{3}}}\bigg \}
\end{multline}
\vspace{-0.5\baselineskip}
 Then, for the model-threshold pair $(\overline{{\bm \theta}}_T,\overline{c}_T)$ provided by Algorithm \ref{alg:fedcvar} after $T$ rounds, we have 
\begin{multline}\label{eq:e_op_bound}
 \mathcal{E}_{opt}    
 \leq     \frac{2\big(\frac{(1-\epsilon)}{\rho}(\beta + \frac{2}{\gamma} G^2) + \epsilon \beta  \big) ({M^2 +B^2})}{\tau T} 
        + \frac{2\sigma \sqrt{{M^2 +B^2}}}{\sqrt{|\mathcal{K}|  \tau T}} 
        + \frac{(1-\epsilon)\gamma}{\rho} 
         \\  
     +  \bigg ( \frac{\big(\frac{(1-\epsilon)}{\rho}(\beta + \frac{2}{\gamma} G^2) + \epsilon \beta  \big)  ({M^2 +B^2})^2}{ T^2} \bigg)^\frac{1}{3} \bigg( 5(\frac{\sigma^2}{\tau})^{\frac{1}{3}}  + 19 { \mu^{\frac{2}{3}} } \bigg)
\end{multline} 
\end{lemma}

\paragraph{Interpretation of Lemma \ref{lemma:multi_fedcvar_convergence}} 

For $\tau=1$, our algorithm finds a model-threshold pair $(\overline{\bm \theta}_T,\overline{c}_T)$ after $T$ communication rounds that guarantees an optimization error of order $O\bigg(\frac{1}{\tau T}+\frac{\sigma}{\sqrt{| \mathcal{K}|\tau T}} + \frac{(1-\epsilon)\gamma}{\rho} \bigg)$. The first term corresponds to the deterministic convergence and the second term refers to the standard statistical noise term encountered by any algorithm that uses $|\mathcal{K}|\tau T$ total stochastic gradients. The third term depends on how accurately the smooth plus function approximates the plus function. When $\gamma$ is sufficiently small, we recover the upper bound for synchronous SGD \cite{advancesFL2}. For $\tau>1$, the last two terms in Eq. \ref{eq:e_op_bound} appear, leading to an optimization error diminishing at a rate of $O\big( T^{-\frac{2}{3}}\big)$.

The guarantees also establish that in the presence of high data heterogeneity (i.e., $\mu \gtrsim \sigma $), the maximum number of local steps we can perform is $\tau=O(|\mathcal{K}|^{-1}(|\mathcal{K}|\tau T)^{\frac{1}{4}})$.  
When there is no heterogeneity, the local rounds increase to $\tau=O(|\mathcal{K}|^{-2} (|\mathcal{K}|\tau T)^{\frac{1}{2}})$. Therefore, for an appropriate selection of local rounds we can handle the error induced by the data heterogeneity across clients when we adopt multiple rounds in \emph{lieu} of a single round per client.

\subsection{Expected Excess Risk} 
Next, we characterize the excess risk, given by  
\begin{equation*}
    \mathcal{E}_{r}=  \mathop{\E}\limits_{\mathcal{A},D}\bigg[\mathop{\E}\limits_{K}\big[\mathop{\E}\limits_{Z|K=k}[{f}(\overline{  \bm \theta}_T,\overline{c}_T;Z)]\big]\bigg]  -
\mathop{\E}\limits_{D}\bigg[\sum\limits_{z \in D} \frac{{f}( \bm \theta^*_D,c^*_D;z)}{n} \bigg], 
\end{equation*}
where $(\overline{ \bm \theta}_T,\overline{c}_T)$ is the average model-threshold pair given by Algorithm \ref{alg:fedcvar} using dataset $D$ and $(\bm \theta^*_D,c^*_D)$ is the optimal solutions pair that minimizes the smoothed empirical objective in Eq. \ref{smooth_empirical_Fed_cvar_tradeoff} for the given dataset $D$. The outer expectation of the first term is taken over the randomness induced by our algorithm $\mathcal{A}$ and of samples $D$, and in the second term with respect to the samples $D$. {The excess risk measures the difference between the expected population risk computed using the produced $(\overline{ \bm \theta}_T,\overline{c}_T)$ and the expected minimum empirical risk given by the empirical optimal pair $( \bm \theta^*_D,c^*_D)$.}

The following lemma -- which relies on the excess risk analysis for stochastic gradient methods in \cite{pmlr-v48-hardt16} -- offers a characterization of $\mathcal{E}_{r}$. 
The proof is available in Appendix \ref{appdx:excess_risk}.  
\begin{lemma}
[Excess Risk Analysis]\label{lemma:excess_risk}
Let assumptions \ref{ass:loss_function} and \ref{ass:convex_sets} hold. Let also the learning rate $\eta =\sqrt{n\big ( {\sum\limits_{k \in \mathcal{K}} b_k} \big)}\frac{\sqrt{M^2 +B^2}}{G_{\rho,\epsilon}\sqrt{T(n+2T)}}$, $\tau=1$, and $\gamma= \frac{2 G_{\rho,\epsilon}^2}{(1-\epsilon +\epsilon \rho)^2}\eta$. Then, for $T$ communication rounds of Algorithm \ref{alg:fedcvar} that satisfy
\vspace{-0.1in}
\begin{equation*}\label{rounds_bound}
   n \bigg ( {\sum\limits_{k \in \mathcal{K}} b_k} \bigg)\big (M^2 +B^2\big ) \bigg (\frac{\beta(1+\epsilon\rho) }{\rho G_{\rho,\epsilon}}\bigg )^2\leq {T(n+2T)},
\end{equation*} 
we have that
\begin{equation*}    
\mathcal{E}_{r}\leq  \frac{G_{\rho,\epsilon}  \sqrt{ {(M^2 +B^2)\big(\frac{2}{n}+\frac{1}{T}\big)} }  }{\sqrt{  {\sum_{k \in \mathcal{K}} b_k}  }}  +\frac{(1-\epsilon)\gamma}{\rho}. 
\end{equation*}
\end{lemma}
\vspace{-4pt}

\paragraph{Interpretation of Lemma \ref{lemma:excess_risk}:} The bound in Lemma \ref{lemma:excess_risk} indicates that, for a fixed step-size $\eta$ and for choices of $\gamma$ and $T$ that satisfy the conditions stated above, our algorithm produces a pair $(\overline{\bm  \theta}_T,\overline{c}_T)$ that yields an excess risk behaving as $O\bigg(\frac{1}{\sqrt{\sum_{k \in \mathcal{K}} b_k}}\sqrt{ \frac{2}{n}+\frac{1}{T} }\bigg)$.  
This bound shows how to effectively improve the overall performance by balancing the trade-off between optimization and generalization, since excess risk can be decomposed into a stability term\footnote{We use algorithmic stability (see Definition \ref{def:uniform_stability} ) to control the generalization error.} and an empirical optimization error term (see for example \cite{https://doi.org/10.48550/arxiv.1804.01619}). Thus, we can directly get from Lemma \ref{lemma:excess_risk} 
that FedSRCVaR has uniform stability of $\zeta\leq \frac{T G^2_{\rho,\epsilon}\eta}{n\sum\limits_{k \in \mathcal{K}}b_k}$. In contrast to the optimization error, the stability term scales with the communication rounds. 
{For $T=n$, the result of Lemma \ref{lemma:excess_risk} 
is of order $O\bigg(\frac{1}{\sqrt{{T\sum_{k \in \mathcal{K}} b_k}}}\bigg)$, and decreases with the square root of communication rounds times the total of batch size. Additionally, if we also have $T= \sum\limits_{k \in \mathcal{K}} b_k$ this quantity further improves and becomes $O\big(\frac{{1}}{T}\big)$. On the other hand, for $T\to \infty$, our bound is of order $O\bigg(\frac{1}{\sqrt{{n\sum_{k \in \mathcal{K}} b_k}}}\bigg)$, indicating that the excess risk scales down with square root of the number of data samples times the total batch size, meaning that we need a large number of client samples to reduce the excess risk. Moreover, if we also pick $n= \sum\limits_{k \in \mathcal{K}} b_k$ we can yield a bound that behaves as $O\big(\frac{{1}}{n}\big)$. We note, however, that $T\to \infty$ creates a communication bottleneck in federated learning systems, since there is a considerably large amount of messages that are exchanged between clients and server.

%% file: Section7.tex
We empirically demonstrate the advantages of the proposed approach on four datasets: (a) \textit{eICU} \cite{Pollard2018}, a dataset with records from various medical centres that we use to predict patient mortality. The data is distributed to $11$ clients and each client is mapped to a single hospital in the dataset. (b) \textit{ACS Employment} \cite{ding2021retiring} for employment classification based on $14$ input features. The data is assigned to $51$ clients based on geolocation.
(c) \textit{MNIST} \cite{deng2012mnist}, a grayscale image dataset that we use to classify $10$ handwritten digits where each digit is allocated to a client. (d) \textit{Celeb-A} \cite{liu2015faceattributes}, a dataset with facial images from celebrities. The target task is gender prediction and the data is randomly assigned to two clients.  

For all methods, the \textit{worst group} refers to the subset of test samples with losses higher than the $(1-\rho)$-quantile of the empirical test loss distribution. The \textit{best group} comprises the remaining test samples.
We measure \textit{utility/mean} as the average risk across all test samples and define the \textit{group risk disparity} as the risk difference between the best and worst groups. Further experimental details and additional experiments are provided in Appendix \ref{appdx:experiments}.
\subsection{Comparison to ML and FL Baselines}
To our knowledge, this is the first work that addresses fairness without access to sensitive groups in FL settings. Hence, we compare our approach with (a) the centralized ML fairness baselines DRO and BPF that aim to achieve fairness without relying on group information, but also with ERM that optimizes for utility; and (b) the FL approaches AFL \cite{AFL} which ensures client fairness, and FedAvg \cite{fedavg} which optimizes for utility, disregarding fairness. We note that the DRO and BPF are the hardest baselines to compare with, since they use a centralized dataset and can obtain the optimal solution. We report our results in Figure \ref{fig:FL_centralized_baselines}.

\begin{figure}
\centering\includegraphics[width=\columnwidth]{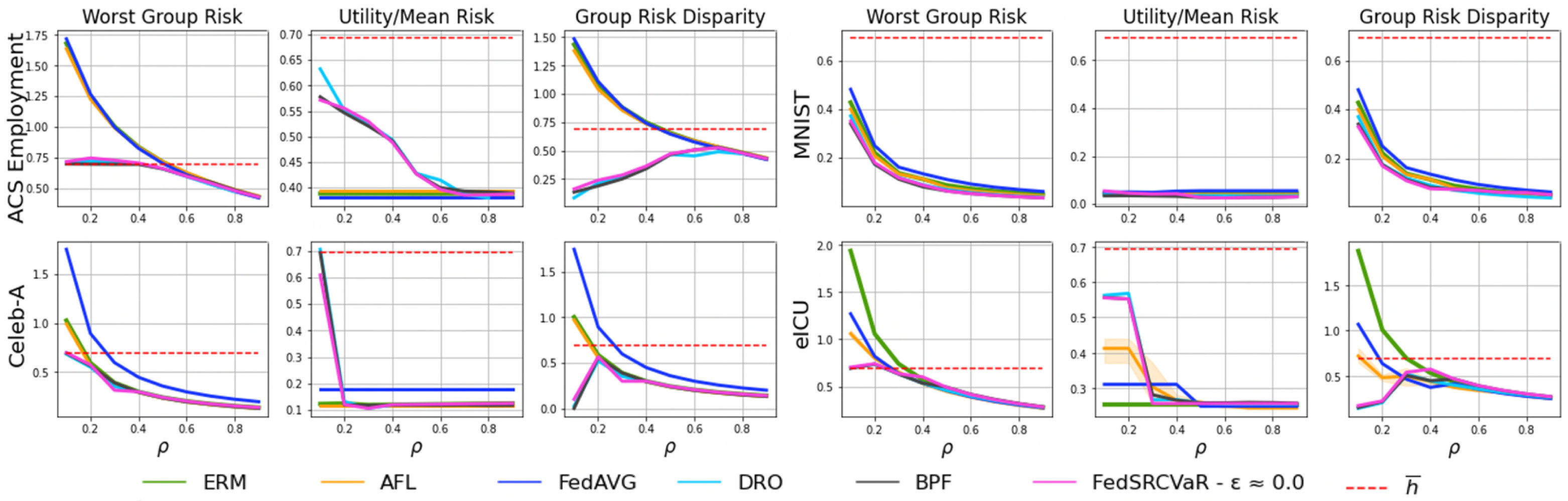}
    \vspace{-1.\baselineskip}
    \caption{Comparison of worst group risk, utility risk and group risk disparity between the best and worst groups on different datasets.
    {$\overline{h}$ denotes the uniform classifier.} FedRCVaR recovers solutions equivalent to centralized settings, while improving both utility and fairness compared to FL baselines in many settings. 
    $\rho$ is a hyperparameter of FedSRCVaR, DRO and BPF. Differences in average performance as a function of $\rho$ for ERM, FedAVG and AFL are due to the variation of the training hyperparameters, since for each $\rho$ we report the hyperparameter combination producing the model with the best performance for each method. 
    }\label{fig:FL_centralized_baselines} \vspace{-0.5\baselineskip}
\end{figure}

For $\epsilon\approx0$, FedRCVaR provides a model with the best performance on the worst group risk, along with BPF and DRO, confirming that our approach effectively produces minimax fair and robust solutions. In some cases, DRO exhibits a higher average risk compared to FedRCVaR and BPF, despite similar performances in the worst group risk, which indicates that DRO underperforms on the remaining population for these cases. 

AFL and FedAvg underperform on the worst group compared to our approach and perform better on the utility task for low $\rho$ values, as expected, since they put different and possibly higher weights on the low-risk samples than FedSRCVaR. We note that AFL maximizes performance over the worst client, while FedSRCVaR optimizes for the worst $\rho$-sized partition across all samples and clients in the federation. We also notice that FedRCVaR improves both worst group fairness and utility performance simultaneously in some datasets, outperforming AFL and FedAvg. This suggests that for small $\epsilon$, minimizing the right-tail risk of the samples is more effective and overall better in handling heterogeneity within the federation but also between training and testing sets. The results of FedSRCVaR and FedAVG vastly vary in most settings except for (a) high worst group size $\rho$ since the worst group will consist of most samples, and/or (b) high $\epsilon$ values for which the utility term has more importance than the fairness term.  For $\rho =1$ our method is theoretically equivalent to FedAvg.

When the produced models are examined on larger values worst group size $\rho$, the risk variance across all different approaches is low, as expected. As we discuss in Remark \ref{rem:critical_rho}, Appendix \ref{appdx:BPF_connections}, there is a critical partition size $\rho$ that leads to the uniform classifier $\overline{h}$ for sufficiently small $\epsilon$ values, which, in conjunction with the generalisation error, justifies the superior performance of $\overline{h}$ for small $\rho$s, as illustrated in our results.

\subsection{Achieving Various Trade-Offs through FedSRCVaR}
We empirically assess the trade-offs FedSRCVaR can accomplish for various combinations of $\epsilon\in \{0.01,0.1,\dots,0.9,1.0\} $ and $\rho\in \{0.1,\dots,0.9\}$ in Figure \ref{fig:trade_offs_xaxis_rho}. The different colours indicate a particular $\epsilon$ value and we report results for models that were trained individually for each pair of $(\epsilon,\rho)$ values. For ACS Employment we distribute the data to $3$ clients based on the race classes: \{Black, White, Others\}.

\begin{figure} [H]
    \centering 
    \includegraphics[width=0.77 \columnwidth]{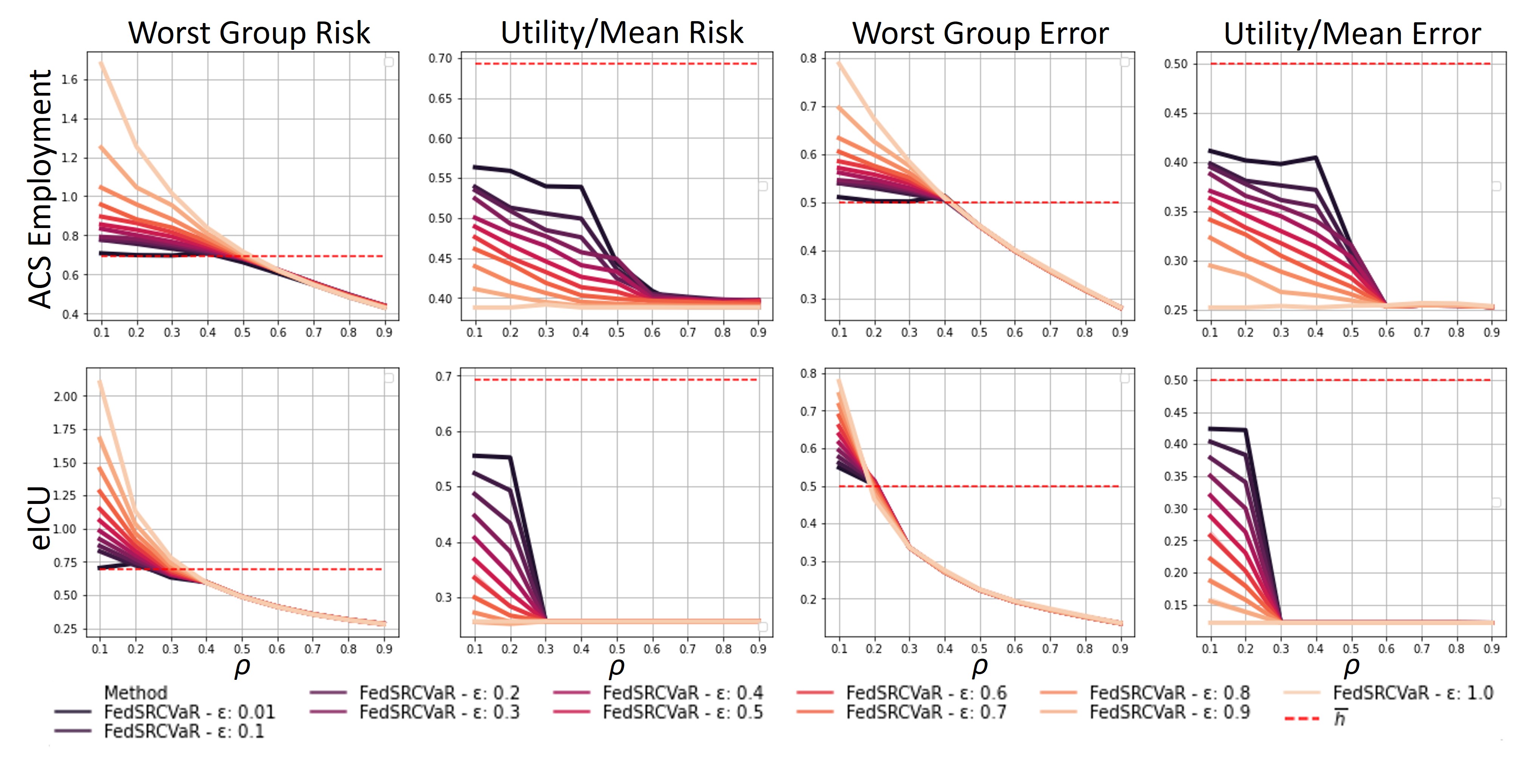}  
    \vspace{-1.1\baselineskip}
\caption{Performance trade-offs among worst group and utility for different pairs of $(\epsilon,\rho) \in \{0.01,0.1,\dots,0.9,1.0\}\times \{0.1,\dots,0.9\}$ values on real datasets. The different colours indicate different $\epsilon$ values. $\overline{h}$ denotes the uniform classifier. A lower score indicates better performance. We report the worst group and average/utility risks, as a function of $\rho$.}\label{fig:trade_offs_xaxis_rho} 
    \vspace{-0.5\baselineskip}
\end{figure}
Figure \ref{fig:trade_offs_xaxis_rho} shows that $\epsilon$ effectively acts as a tuning parameter between worst group fairness and average performance. For small $\rho$ values, $\epsilon$ has a significant impact on the worst group and the utility performance. We observe that the larger the $\epsilon$ the lower the average utility errors and risks, while as we decrease $\epsilon$ we boost the performance on worst group. Note that for $\epsilon \approx 0$ and $\rho \approx 0$, the worst-group risk is close to the uniform classifier risk which is consistent with Remark \ref{rem:critical_rho} in Appendix \ref{appdx:BPF_federalization}, and conclusions drawn about the existence of a critical worst-group size under which we yield the uniform classifier in \cite{pmlr-v139-martinez21a}. 
On the other hand, for large $\rho$s we notice that all solutions are equivalent and the parameter $\epsilon$ has almost no influence on the solution. Interestingly, for particular values of $\epsilon$ and $\rho$, FedSRCVaR can recover client robustness solutions (akin to AFL), even though our objective does not explicitly aim for that.

\subsection{FedSRCVaR with Multiple Local Rounds}
In Figure \ref{fig:multi_roundexps} we compare the performance of FedSRCVaR for $\tau \in \{1,5,10\}$, and FedAVG, on the ACS Employment and eICU datasets. 
For small $\rho$ values, FedSRCVaR for $\tau \in \{5,10\}$ exhibits a higher worst-group risk compared to the FedSRCVaR with $\tau=1$. This suggests that conducting multiple local rounds may result in inferior performance for the worst group when $\rho$ is small, as indicated in our convergence guarantees. Moreover, as $\rho$ increases, FedSRCVaR for $\tau \in \{5,10\}$ demonstrates improvement in worst-case fairness similar to FedSRCVaR with $\tau=1$. This implies that the impact of performing additional local epochs becomes less significant as the worst-group size becomes larger. For sufficiently large values of $\rho$, both methods converge to the same solution as FedAVG. 
 
\begin{figure} [H] 
\centering\includegraphics[width=0.74\columnwidth]{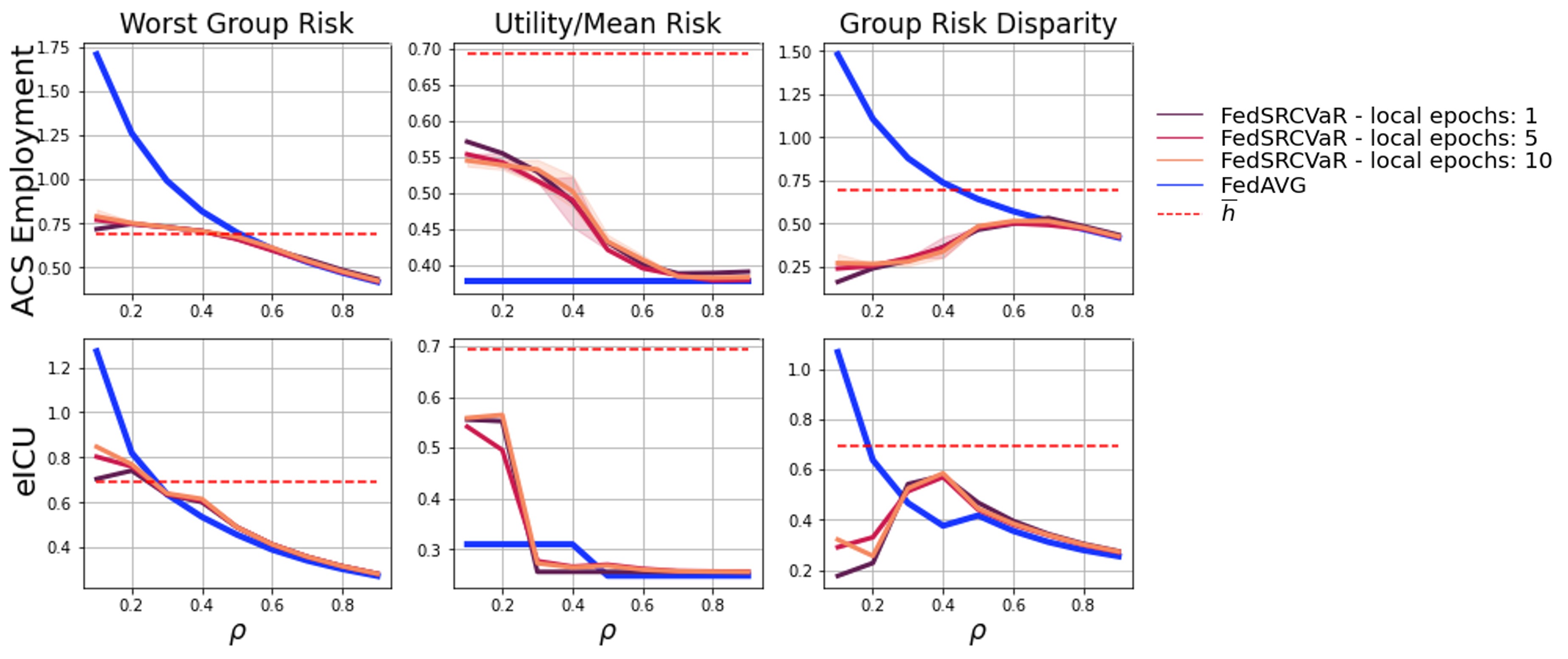}  
\vspace{-1.4\baselineskip}
\caption{Performance comparison between FedSRCVaR for local epochs $\tau \in \{1,5,10\}$ and FedAVG. $\overline{h}$ denotes the uniform classifier. We report the worst group and average/utility risks, and the group risk disparity between the worst-performing samples and the remaining population, as a function of $\rho$.} 
\label{fig:multi_roundexps} 
\end{figure}

%% file: Section8.tex
Federated learning is crucial to obtaining large, representative datasets across every sensitive group. Ensuring fairness across protected groups is essential for responsible machine learning, but prior knowledge of such groups is not always available, due to privacy constraints and evolving fairness requirements. 

This is the first work to present a flexible federated learning objective to ensure minimax Pareto fairness with respect to any group of sufficient size. We propose an algorithm that solves a proxy of the proposed objective, providing performance guarantees in the convex setting.
Experimentally, our approach surpasses relevant FL baselines, exhibits comparable performance to centralized ML approaches, and demonstrates the ability to achieve a diverse range of solutions. For a single local epoch, FedSRCVaR is robust to data imbalances and heterogeneity across clients, yielding the same solution as centralized ML settings. For multiple local epochs, FedSRCVaR improves communication costs, but might yield a suboptimal solution when data are highly non-iid across clients.

%% file: appendix/appendix_main.tex
\input{appendix/BPF_connections}

\input{appendix/basic_definitions}

\input{appendix/algotithmic_details_new}

\input{appendix/experimental_details}

%% file: appendix/BPF_connections.tex
\section{Formal Connection to Pareto Subgroup Robustness from Centralized ML}\label{appdx:BPF_connections}

An existing approach to fair centralized machine learning through subgroup robustness is BPF \cite{martinez2020}. In this Appendix, we explore the relationship between the proposed RCVaR objective and BPF objective, by showing via the Lagrange dual that one can recover RCVaR from the BPF. 

To establish this connection, we consider the setting described in Section \ref{preliminaries}.
We also introduce an additional random variable $G$ which indicates whether a certain input-target pair belongs to the worst-performing group, i.e., the group with the high-risk samples. We let $G=1$ denote the samples belonging to the worst performing group and $G=0$ for the remaining data that do not belong to the worst group. We refer to the group $G=0$ as the best-performing group. 

By setting the worst group size to be equal to the probability $\rho$, i.e., $p(G=1)=\rho$, and constraining  $p(G=1|Z)> \epsilon$, the BPF objective is defined as 
\begin{align}\label{BPF_objective}
      \min\limits_{h \in \mathcal{H}}\max\limits_{\substack{p(G=1|Z) \\ \textnormal{s.t. } p(G=1)= \rho \\ p(G=1|Z)> \epsilon}}  \mathop{\mathbb{E}}\limits_{Z \sim p(Z)}\bigg [\frac{p(G=1|Z)}{p(G=1)}{L_{h,Z}}\bigg] = 
      \min\limits_{h \in \mathcal{H}}\max\limits_{\substack{\lambda(Z) \in Q_{\epsilon,\rho} \\ \textnormal{s.t. }\bigintsss_{\mathcal{Z}} \lambda(z) dz=1 }}  \bigintsss_{\mathcal{Z}} \lambda(z) \ell(h;z) dz ,     \tag{BPF}
\end{align}
where  
$Q_{\epsilon,\rho}= \bigg\{ \lambda(\cdot) : \lambda(z) \in \big[\frac{p(z)}{ \rho}\epsilon ,\frac{p(z)}{ \rho}\big] \bigg\}$ and $\lambda(Z) \in Q_{\epsilon,\rho}$ is the density of variable $Z$.  
We define the Lagrangian of  \ref{BPF_objective} as
\begin{equation*}
 \textsc{L}_{BPF}( \lambda(Z),\mu) =  \bigintsss_{\mathcal{Z}} \lambda(z) \ell(h;z) dz +\mu^*\bigg( 1- \bigintsss_{\mathcal{Z}} \lambda(z) dz\bigg) \\ =\bigintsss_{\mathcal{Z}} \lambda(z) (\ell(h;z)-\mu^*) dz +\mu^*,
\end{equation*}
where $\mu^*$ is the Lagrange multiplier of the constraint in \ref{BPF_objective}.
For a fixed hypothesis $h \in \mathcal{H}$, the optimal density $\lambda^*(Z)$ satisfies
\begin{equation*}
\begin{array}{rl}
     \lambda^*(Z)=\arg \max\limits_{\lambda(Z) \in Q_{\epsilon,\rho}} \textsc{L}_{BPF}( \lambda(Z),\mu)   
     =  \arg \max\limits_{\lambda(Z) \in Q_{\epsilon,\rho}}  \bigintsss_{\mathcal{Z}} \lambda(z) (\ell(h;z)-\mu^*) dz   =   \begin{cases}
                    \frac{p(z)}{ \rho}, & \text{if $\ell(h;z)>\mu^*$}\\ 
                \frac{p(z)}{ \rho}\epsilon, &  \text{if $\ell(h;z)\leq\mu^*$} 
        \end{cases}
\end{array}    
\end{equation*}
Furthermore, using the fact that $\bigintsss_{\mathcal{Z}} \lambda(z) dz=1$, 
we can compute the Lagrange multiplier $\mu^*$ as follows

\begin{align*}
     \bigintsss_{\ell(h;z)\leq\mu^*} \frac{\epsilon}{\rho}{p(z)} dz +
     \bigintsss_{\ell(h;z)>\mu^*} \frac{p(z)}{ \rho} dz =  1 
     \Longrightarrow \frac{\epsilon}{\rho} \bigintsss_{\mathcal{Z}} {p(z)} dz + \bigg(\frac{1-\epsilon}{\rho}\bigg)
     \bigintsss_{\ell(h;z)>\mu^*} {p(z)} dz =  1 \\ \\ 
     \Longrightarrow  \bigintsss_{\ell(h;z)>\mu^*} p(z) dz =  \frac{\rho -\epsilon}{(1-\epsilon)}
      \Longrightarrow  \bigintsss_{\ell(h;z)\leq\mu^*} p(z) dz =  1- \frac{\rho -\epsilon}{(1-\epsilon)}     \Longleftrightarrow \mu^*=F^{-1} (1-  \rho' ),    \tag{w/ $\rho'=\frac{\rho -\epsilon}{(1-\epsilon)}$}
\end{align*}
{where $F^{-1}(\cdot)$ corresponds to the inverse of  $F(\cdot)$, and $F(L_{h,Z})$ is the cumulative distribution function of $L_{h,Z}$.} 

Then, by substituting the optimal density $\lambda^*(Z)$ and $\mu^*$ in the \ref{BPF_objective} objective we get
\begin{equation} \label{BPF_RCVAR_equivalence}
 \begin{array}{rl} 
     & \min\limits_{h \in \mathcal{H}}\bigintsss_{\ell(h;z)\leq\mu^*} \epsilon\frac{p(z)}{\rho} \ell(h;z) dz + \bigintsss_{\ell(h;z)>\mu^*} \frac{p(z)}{\rho} \ell(h;z) dz  
    \\ \\=&  \min\limits_{h \in \mathcal{H}} \frac{\epsilon}{\rho} \bigintsss_{\mathcal{Z}} {p(z)} \ell(h;z) dz + \frac{1-\epsilon}{\rho}
     \bigintsss_{\ell(h;z)> F^{-1} (1- \rho') } p(z) \ell(h;z) dz \quad \text{(Recall $\rho= \bigintsss_{\ell(h;z)> F^{-1} ( 1- \rho') }p(z) dz$)}
\\\\= & \min\limits_{h \in \mathcal{H}} \frac{\epsilon}{ \rho} \mathop{\mathbb{E}}\limits_{Z \sim p(Z)}[L_{h,Z}] + (1-\epsilon)\mathop{\mathbb{E}}\limits_{Z \sim p(Z)} \big[ L_{h,Z}|L_{h,Z}> F^{-1}\big( 1- \rho' \big) \big] 
\\\\= &\min\limits_{h \in \mathcal{H}} \frac{\epsilon}{ \rho} \mathop{\mathbb{E}}\limits_{Z \sim p(Z)}[L_{h,Z}] +  (1-\epsilon) CVaR_{1-\rho'}(L_{h,Z})\\\\
     = &\min\limits_{h \in \mathcal{H}} \underbrace{{\epsilon} \mathop{\mathbb{E}}\limits_{Z}[L_{h,Z}] +  (1-\epsilon) CVaR_{ 1- \rho'}(L_{h,Z}) }_{\mathrm{RCVaR \ for \  \rho'=  \frac{\rho -\epsilon}{(1-\epsilon)} \ and \ trade-off \ \epsilon \approx 0 }} + \underbrace{\frac{\epsilon(1-\rho)}{ \rho} \mathop{\mathbb{E}}\limits_{Z}[L_{h,Z}] }_{\mathrm{Add'l \ term}}  
   \\ \\ \approx &\min\limits_{h \in \mathcal{H}} {{\epsilon} \mathop{\mathbb{E}}\limits_{Z \sim p(Z)}[L_{h,Z}] +  (1-\epsilon) CVaR_{ 1- \rho'}(L_{h,Z}) } \quad \text{(since $\epsilon\approx0$)}
 \end{array}
\end{equation}

Eq. \ref{BPF_RCVAR_equivalence} shows that BPF is equivalent to RCVaR plus an additional error term for the same $\epsilon$ and $\rho'=  \frac{\rho -\epsilon}{(1-\epsilon)} $. We remark that for sufficiently small $\epsilon$ values (i.e., $\epsilon \approx 0$)
 the additional term is negligible and probability $\rho'  \approx \rho$, which makes the two objectives equivalent. 
 
 Due to this equivalence, we argue that the results presented in Lemma 3.2 in \cite{martinez2020}, which proves the existence of a critical partition size $\rho$ that leads to the uniform classifier for sufficiently small $\epsilon$ values; and Lemma 3.3 in \cite{martinez2020}, that studies the penalty in performance when we use subgroup robustness instead of known groups, apply also to our case for $\epsilon\approx0$. 
In Remark \ref{rem:critical_rho}, we state the impact of the threshold $c$ in the resulting hypothesis in our objective, RCVaR.

\begin{remark}\label{rem:critical_rho}
The hypothesis that determines the ($1-\rho$)-quantile $c$, for which any realizations of ${L_{h,Z}}$ are at most $c$, is the uniform classifier $\overline{h}: \overline{h}_y(X)=\frac{1}{|\mathcal{{Y}}|} \forall y \in \mathcal{Y}$. \end{remark}

\subsection{Federalization of BPF}\label{appdx:BPF_federalization}

One of the merits of RCVaR compared to BPF is that it can easily be deployed in dynamic machine learning settings, such as online learning settings, where we (continue to) optimize the global model using a stream of new data arriving sequentially in real-time both in centralized and federated learning settings. 
In contrast, deploying BPF may present complexities and limitations due to the need for estimating and optimizing per-sample adversarial weights at each optimization round. Since data arrives sequentially, managing and accessing the last risk evaluation for every sample and adjusting the set from which adversarial weights are selected become computationally expensive. These factors can impede the efficiency of the learning process.

Another advantage of RCVaR is its suitability for federated learning. In Algorithm \ref{alg:fedcvar} we showed that the federalization of RCVaR requires only the exchange of the updated model-threshold pair between clients and the server. This efficient exchange allows for achieving a solution equivalent to our centralized BPF objective in \cite{pmlr-v139-martinez21a}. In contrast, federalizing BPF poses challenges, such as (a) failure to provide guarantees that the produced global model is equivalent to centralized BPF; or (b) significantly amplifying the computation and communication costs of the federated learning procedure, and raising additional privacy concerns. 

In particular, one approach to federalize BPF is to assign a common value of $\rho$ across all clients (i.e., $\rho=\rho_k \forall k \in \mathcal{K}$), resulting in an overall partition of size $\rho$. However, this approach has a weaker adversary compared to our proposed framework. In our framework, the adversary has the flexibility to choose any partition of size $\rho$ and allocate more weight to clients with worse performances by shifting their budget ($\rho_k$) from clients with higher utility. Also, by assigning a fixed $\rho=\rho_k$, we have no guarantee that the solution is equivalent to (centralized) BPF, while our proposal is guaranteed to produce an equivalent solution.

Another way to federalize BPF would be for clients to share information about their local loss distributions with the server. This requires the clients to know their local group sizes (i.e., $\rho_k$) to ensure global Pareto subgroup robustness. However, due to the data heterogeneity across clients, there is no guarantee or prior knowledge about how the global worst group is distributed across clients at each communication round $t$. 
Thus, acquiring this information would involve additional communication rounds, increased computations on the client side, and additional computation on the server side to correctly hash and share the local group sizes. This not only amplifies the computation and communication costs but also raises privacy concerns as clients need to share information about their local loss distributions.

%% file: appendix/basic_definitions.tex
\section{Basic Definitions}\label{appdx:definitions}
We first provide some standard definitions and remarks.  
In what follows, the norm $||\cdot||$ denotes the Euclidean norm.

\begin{definition}[Convex Set]\label{def:set_convexity}
A set ${\Theta}$ is convex if for any two points $\theta_1,\theta_2 \in {\Theta}$ we have that their convex combination
\begin{equation*}
    \theta=\mu \theta_1 +(1-\mu)\theta_2,\textnormal{ with } \mu \in [0,1],
\end{equation*}
also belongs to $ {\Theta}$, i.e. $\theta \in {\Theta}$.
\end{definition}

\begin{definition}[Convex Function]\label{def:convexity_simple}
A function $f$ with domain $dom(f)$ is convex if and only if $dom(f)$ is a convex set and for all $v,w \in dom(f)$ we have that 
\begin{equation}\label{jensens_ineq}
    f(\mu w+ (1-\mu)v) \leq \mu f( w)+ (1-\mu)f(v),\textnormal{ with } \mu \in [0,1]. 
\end{equation}
\end{definition} 
\begin{definition}[Lipschitzness]\label{def:lipschitzness}
A function $f: dom(f) \to \R$ is $G$-Lipschitz if for all $ v,w \in  dom(f)$ we have that
\begin{equation}\label{def_lipschitz}
   ||\nabla f(w) ||\leq G \quad \textnormal{and } \quad | f(w)-  f(v)| \leq G ||w-v||, \textnormal{ for some }  G>0. 
\end{equation}
\end{definition}

{ 
\begin{lemma}[First Order Condition]\label{def:convexity}
A differentiable function $f$ with domain $dom(f)$ is convex if $dom(f)$ is convex and for any $v,w \in dom(f)$ we have that 
\begin{equation*}
    f(w) \leq f(v)+ \langle \nabla f(v), w-v\rangle.
\end{equation*}
\end{lemma}
\begin{proof}
For proof see section 3.1.3 in \cite{boyd2004convex}.
\end{proof}
}

\begin{definition}[Smoothness]\label{def:smoothness}
A function $f:dom(f) \to \R$ is $\beta$-smooth if it is continuously differentiable and its gradient is $\beta$-Lipschitz, i.e. $\exists \beta: \forall v,w \in dom(f)$ we have that 
\begin{equation*}
    ||\nabla f(w)-\nabla f(v)|| \leq \beta ||w-v||.
\end{equation*} 
\end{definition}

\begin{definition}[Sub-gradient]\label{def:subgradient}
Let $f:dom(f) \to \R$, with $dom(f)\subseteq \R^d$. Then $g \in \R^d$ is a subgradient of $f$ at point $v$ if for any $w \in dom(f)$ we have that 
\begin{equation*}
    f(w)-  f(v) \geq \langle g, w-v\rangle. 
\end{equation*}
\end{definition}

Note that the subgradient $g$ might not be unique. We denote $\partial f(v)$ the set of subgradients computed at a point $v$, also called subdifferential of $f$ at $v$, where $g \in \partial f(v)$. We also note that when a function $f:dom(f) \to \R$ is $G$-Lipschitz and convex, Eq. \ref{def_lipschitz} becomes $||g||\leq G$. We provide this elementary proof in Lemma \ref{lemma:Lipschitz_subgradient}.

\begin{lemma}\label{lemma:Lipschitz_subgradient}
Let a function $f:dom(f) \to \R$ be a $G-$Lipschitz continuous and $\partial f(v) \neq \emptyset$. Then, for any $v \in dom(f)$ we have that 
\begin{equation}\label{eq:Lipschitz_subgradient}
     ||g|| \leq G, \quad \textnormal{with}  
     \quad 
      g \in \partial f(v)  
\end{equation}
\end{lemma}
\begin{proof}
    Since $f$ is $G-$Lipschitz we have that
\begin{align*}
    | f(w)-  f(v)| \leq G ||w-v||, \textnormal{ for some }  G>0
\end{align*}
 Also from the subgradient definition we know that
 \begin{align*}
      f(w)-  f(v) \geq \langle g, w-v\rangle.
 \end{align*}
Combining the two inequalities we get that $||g|| \leq G$.
 
\end{proof}

{
\begin{definition}[Uniform Stability, \cite{bousquet2002stability}]\label{def:uniform_stability} Let $f:dom(f) \to \R$.
A randomized algorithm $\mathcal{A}$ is $\zeta-$uniformly stable if for any datasets $D,D'$ that differ in at most a single sample, we have that
\begin{equation}\label{eq:def_stability}
    \sup\limits_{z} \mathop{\E}[f(\mathcal{A}(D);z)-f(\mathcal{A}(D');z) ]\leq \zeta
\end{equation}
where $\zeta>0$ and the expectation is w.r.t. the randomness of the algorithm and the samples. 
\end{definition}

%% file: appendix/algotithmic_details_new.tex
\section{Analysis of Algorithm \ref{alg:fedcvar}}\label{appdx:algo}

In this section, we analyze the properties of Algorithm \ref{alg:fedcvar}, as stated in Lemmas \ref{lemma:multi_fedcvar_convergence} and \ref{lemma:excess_risk}.

\input{appendix/smooth_objective}

\subsection{Convergence of Algorithm \ref{alg:fedcvar}}\label{appdx:convergence}

\input{proofs/convergence_multirounds}

\subsection{Excess Risk Analysis of Algorithm \ref{alg:fedcvar}
}\label{appdx:excess_risk} 
In order to derive an upper bound in excess risk $\mathcal{E}_r$, we use the results for stochastic gradient methods from Proposition 5.4. in \cite{pmlr-v48-hardt16}, which we also repeat using our notation in Lemma \ref{lemma:hardt_convergence} for convenience.

\begin{lemma}[Proposition 5.4. in \cite{pmlr-v48-hardt16}]\label{lemma:hardt_convergence} Let assumptions \ref{ass:loss_function} and \ref{ass:convex_sets}
hold. Let also $\overline{\bm \theta}_T= \sum\limits_{t=1}^T \bm \theta_t$ and $\bm \theta^{\star}_D=\arg\min\limits_{\bm \theta \in \Theta } \frac{1}{n} \sum\limits_{i=1}^n \ell(\bm \theta;z_i)$. Suppose we run $T$ steps of SGD with a learning rate $\eta=\frac{M\sqrt{n}}{G\sqrt{T(n+2T)}}$. Then,
\begin{equation}\label{eq:hardt_bound}
   \E\bigg[ \mathop{\E}\limits_{Z \sim p(Z)}[\ell(\overline{\bm \theta}_T;Z)] -   \frac{1}{n} \sum_{i=1}^n \ell(\bm \theta^{\star}_D;z_i)
     \bigg] 
    \leq \frac{1}{2} \bigg(\frac{M^2}{\eta T} + \frac{G^2 \eta}{n}(2T+n) \bigg) 
\end{equation}
where the outer expectation is taken w.r.t. the internal randomness of the algorithm and the randomness of samples $D$.  
\end{lemma} 
{Note that the selected step-size in Lemma \ref{lemma:hardt_convergence} satisfies $\eta \leq \frac{2}{\beta}$ (see Theorem 3.7 in \cite{pmlr-v48-hardt16} for more details). } This condition is required to be satisfied in our analysis as well. Next, we apply these results in our setting and provide the formal proof of Lemma \ref{lemma:excess_risk}.
\input{proofs/lemma_sgd_opt}

%% file: appendix/smooth_objective.tex
\subsection{Smooth Approximation of Eq. \ref{empirical_Fed_cvar_tradeoff}
}\label{appdx:smooth_approx}

In order to provide an algorithmic analysis for our setting, we require the auxiliary function $f$, defined as 
\begin{equation*}
    f(\bm \theta,c;z)= (1-\epsilon)[c + \frac{1}{\rho}(\ell(\bm \theta;z) - c)_+] + {\epsilon} \ell(\bm \theta;z)
\end{equation*}
to be smooth. For this reason, we define 
the smoothed version of the auxiliary function $f$ as
\begin{equation*}
    \Tilde{f}(\bm \theta,c;z)= (1-\epsilon)[c + \frac{1}{\rho}s(\ell(\bm \theta;z) - c)] + {\epsilon} \ell(\bm \theta;z),
\end{equation*}
where $s$ is a convex and $(\frac{2}{\gamma})-$smooth function.

Given a function $s$ that satisfies the conditions given in Definition \ref{def:smoothed_plus_function}, 
we provide the properties for the auxiliary function $f$ and the smoothed function $\Tilde{f}$ in \ref{lemma:properties}.

\begin{proof}[{Proof of Lemma \ref{lemma:properties}}]
We use the same numbering to prove each case.
\paragraph{1.}{ 
For the convexity of $f$ we just need to show that $(\cdot)_+$ is convex. For a fixed $j \in \{1,\dots,m\}$, with $m>0$, and $ \lambda \in [0,1]$, we have that
\begin{align*}
&y_j \leq \max_i y_i,\quad x_j \leq \max_i x_i \\&\textnormal{and thus}\quad  \lambda x_j + (1-\lambda)y_j \leq \lambda \max_i x_i + (1-\lambda)\max_i y_i.
\end{align*}
Consequently, we also have $\max\limits_j [\lambda x_j + (1-\lambda)y_j]\leq \lambda \max\limits_i x_i + (1-\lambda)\max\limits_i y_i.$ {Note that in our scenario $m=2$ and $y_j\in\{0,\ell(\bm \theta, z_i)-c\}$ and $x_j\in\{0,\ell(\bm \theta, z_l)-c\}$ with $j \in [m]$ and $i,l \in [n]$.}
 
Note that the smoothed plus function $s$ is convex w.r.t. $z$ by definition. Thus, the convexity of the function $\Tilde{f}$ is immediate since $\Tilde{f}$ is a linear combination of convex terms.} 

\paragraph{2.}For the second property, we let $g$ denote the subgradient of $f$ for a fixed pair of $({\bm \theta},c)$ (i.e. $g \in \partial_{({\bm \theta},c)} f({\bm \theta},c;z)$).
As we see in the previous appendix, the Euclidean norm of the subgradient of a convex and $G_{\rho,\epsilon}-$Lipschitz function, is upper bounded by $G_{\rho,\epsilon}$, i.e. $||g||\leq G_{\rho,\epsilon}$. Thus, we work out a Lipschitzness parameter by finding an upper bound for the subgradient of $f$, $\quad \forall g \in \partial_{({\bm \theta},c)} f({\bm \theta},c;z)$.

We remark that the plus function $(\cdot)_+$ in function $f$, induces three scenarios for any $z$: (i) $\ell({\bm \theta};z)>c$, (ii) $\ell({\bm \theta};z)=c$, or (iii) $\ell({\bm \theta};z)<c$. Thus, we define the set of subgradients as

    \begin{align}\label{set_subgrads}
        \partial_{({\bm \theta},c)} f({\bm \theta},c;z)= & 
        \begin{cases}
                    \begin{bmatrix}
                        \big(\frac{(1-\epsilon)}{\rho}+ \epsilon \big) \nabla_{\bm \theta}\ell({\bm \theta};z)\\ 
                        (1-\epsilon)(1-\frac{1}{\rho})
                    \end{bmatrix},  &\text{if $\ell({\bm \theta};z)>c$}   \\ \\
                    \begin{bmatrix}
                        \big(\frac{(1-\epsilon)t}{\rho}+ \epsilon \big) \nabla_{\bm \theta}\ell({\bm \theta};z)\\ 
                        (1-\epsilon)(1-\frac{t}{\rho})
                    \end{bmatrix}, &\text{if $\ell({\bm \theta};z)=c$, $t\in [0,1]$}   \\ \\
                    \begin{bmatrix}
                        \epsilon  \nabla_{\bm \theta}\ell({\bm \theta};z)\\ 
                         1-\epsilon
                    \end{bmatrix},  & \text{if $\ell({\bm \theta};z)<c$}  
        \end{cases}
    \end{align}

Consequently, $\forall g \in \partial_{({\bm \theta},c)} f({\bm \theta},c;z)$ we get
\begin{align*} 
    ||g||^2\leq  
                 \begin{cases}
                    G^2 \bigg(\frac{(1-\epsilon)}{\rho}+ \epsilon \bigg)^2 + (1-\epsilon)^2(1-\frac{1}{\rho})^2, &  \text{if $\ell({\bm \theta};z)>c$}   \\ \\
                    \max\limits_{t \in [0,1]} \bigg [G^2 \bigg(\frac{(1-\epsilon)t}{\rho}+ \epsilon \bigg)^2 + (1-\epsilon)^2(1-\frac{t}{\rho})^2 \bigg ], &  \text{if $\ell({\bm \theta};z)=c$}   \\ \\
                    G^2  \epsilon^2 + (1-\epsilon)^2, & \text{if $\ell({\bm \theta};z)<c$}   
        \end{cases}
\\
\Rightarrow ||g||\leq  \max \bigg\{ \sqrt{\frac{G^2(1-\epsilon +\epsilon \rho)^2 + (1-\epsilon)^2({\rho}-1)^2 }{\rho^2}} , \sqrt{G^2  \epsilon^2 + (1-\epsilon)^2}\bigg\}
\end{align*}
Using similar reasoning, we can show that the smoothed auxiliary function $\Tilde{f}({\bm \theta},c;z)$ is $G_{\rho,\epsilon}-$Lipschitz for all $z$. Let $s'$ be the derivative of the smoothed plus function $s$. We have that
\begin{equation}
    \nabla \Tilde{f}({\bm \theta},c;z)=\begin{bmatrix}
        \bigg(\frac{(1-\epsilon)s'(\ell({\bm \theta};z)-c)}{\rho}+ \epsilon \bigg) \nabla_{\bm \theta}\ell({\bm \theta};z)\\ 
         (1-\epsilon)\bigg(1-\frac{s'(\ell({\bm \theta};z)-c)}{\rho}\bigg)
        \end{bmatrix}
\end{equation}
Since $\ell \in [0,1]$, we also have that $s' \in [0,1]$ and thus $||\nabla \Tilde{f}||^2\leq \max\limits_{t \in [0,1]} \bigg [G^2 \bigg(\frac{(1-\epsilon)t}{\rho}+ \epsilon \bigg)^2 + (1-\epsilon)^2(1-\frac{t}{\rho})^2 \bigg ]\leq G_{\rho,\epsilon}^2$.

\paragraph{3.}Finally, we recall that by assumption the loss function $\ell$ is $\beta-$smooth and the smoothing plus function $s$ is $\frac{2}{\gamma}-$smooth. Then, for any pairs $m_1=({\bm \theta}_1,c_1)$ and $m_2=({\bm \theta}_2,c_2)$, for the first coordinate of $\nabla \Tilde{f}$ we obtain
\begin{align}\label{smoothness_proof_a}
   & \bigg|\bigg|\bigg(\frac{(1-\epsilon)s'(\ell({\bm \theta}_1;z)-c_1)}{\rho}+ \epsilon \bigg) \nabla_{\bm \theta}\ell({\bm \theta}_1;z)  -\bigg(\frac{(1-\epsilon)s'(\ell({\bm \theta}_2;z)-c_2)}{\rho}+ \epsilon \bigg) \nabla_{\bm \theta}\ell({\bm \theta}_2;z)\bigg |\bigg| \notag\\\notag \\  
    =& \bigg|\bigg|\bigg(\frac{(1-\epsilon)s'(\ell({\bm \theta}_1;z)-c_1)}{\rho}+ \epsilon \bigg) \nabla_{\bm \theta}\ell({\bm \theta}_1;z) - \bigg(\frac{(1-\epsilon)s'(\ell({\bm \theta}_1;z)-c_1)}{\rho}+ \epsilon \bigg) \nabla_{\bm \theta}\ell({\bm \theta}_2;z)\notag
    \\ &+ \bigg(\frac{(1-\epsilon)s'(\ell({\bm \theta}_1;z)-c_1)}{\rho}+ \epsilon \bigg) \nabla_{\bm \theta}\ell({\bm \theta}_2;z)  -\bigg(\frac{(1-\epsilon)s'(\ell({\bm \theta}_2;z)-c_2)}{\rho}+ \epsilon \bigg) \nabla_{\bm \theta}\ell({\bm \theta}_2;z)\bigg |\bigg|\notag 
\\\notag\\
\leq& \bigg(\frac{(1-\epsilon)\big|s'(\ell({\bm \theta}_1;z)-c_1)\big|}{\rho}+ \epsilon \bigg)\cdot ||\nabla_{\bm \theta}\ell({\bm \theta}_1;z)-\nabla_{\bm \theta}\ell({\bm \theta}_2;z)|| \notag
   \\ &+||\nabla_{\bm \theta}\ell({\bm \theta}_2;z)|| \cdot\bigg|\bigg(\frac{(1-\epsilon)s'(\ell({\bm \theta}_1;z)-c_1)}{\rho}+ \epsilon \bigg) -\bigg(\frac{(1-\epsilon)s'(\ell({\bm \theta}_2;z)-c_2)}{\rho}+ \epsilon \bigg)  \bigg|
\end{align}
We know that: 
\begin{enumerate}
    \item Since  $\ell \in [0,1]$, we have that $s' \in [0,1]$ and 
     \begin{equation}\label{codn1}
    \bigg(\frac{(1-\epsilon)\big|s'(\ell({\bm \theta}_1;z)-c_1)\big|}{\rho}+ \epsilon \bigg) \leq \bigg(\frac{(1-\epsilon)}{\rho}+ \epsilon \bigg). 
     \end{equation}
\item The loss function $\ell$ is $\beta-$smooth, i.e.
    \begin{equation}\label{codn2}
     ||\nabla_{\bm \theta}\ell({\bm \theta}_1;z)-\nabla_{\bm \theta}\ell({\bm \theta}_2;z)|| \leq \beta||{\bm \theta}_1-{\bm \theta}_2|| .
\end{equation}
\item The loss function $\ell$ is convex and $G-$Lipschitz, thus
      \begin{equation}\label{codn3}
     || \nabla_{\bm \theta}\ell({\bm \theta}_2;z)|| \break\leq G. 
\end{equation}
\item The smoothed plus function $s$ is $\frac{2}{\gamma}$-smooth, i.e.
    \begin{equation}\label{codn4}
   ||s'(\ell({\bm \theta}_1;z)-c_1)   -s'(\ell({\bm \theta}_2;z)-c_2)|| \leq \frac{2}{\gamma}||  (\ell({\bm \theta}_1;z)-c_1)   - (\ell({\bm \theta}_2;z)-c_2)||
\end{equation}
\end{enumerate}

By substituting \ref{codn1}-\ref{codn4} into Eq. \ref{smoothness_proof_a} we get
\begin{align*} 
 &\bigg(\frac{(1-\epsilon)\big|s'(\ell({\bm \theta}_1;z)-c_1)\big|}{\rho}+ \epsilon \bigg)\cdot ||\nabla_{\bm \theta}\ell({\bm \theta}_1;z)-\nabla_{\bm \theta}\ell({\bm \theta}_2;z)||
     \\&+||\nabla_{\bm \theta}\ell({\bm \theta}_2;z)|| \cdot\bigg|\bigg(\frac{(1-\epsilon)s'(\ell({\bm \theta}_1;z)-c_1)}{\rho}+ \epsilon \bigg)  -\bigg(\frac{(1-\epsilon)s'(\ell({\bm \theta}_2;z)-c_2)}{\rho}+ \epsilon \bigg)  \bigg| 
    \\ \\   
    \leq & \bigg(\frac{(1-\epsilon)}{\rho}+ \epsilon \bigg)\beta|| {\bm \theta}_1 -{\bm \theta}_2 ||  + G\frac{(1-\epsilon)}{\rho} \big|s'(\ell({\bm \theta}_1;z)-c_1)  -s'(\ell({\bm \theta}_2;z)-c_2)\big|\allowdisplaybreaks
    \\ \\  
    \leq & \bigg(\frac{(1-\epsilon)}{\rho}+ \epsilon \bigg)\beta|| {\bm \theta}_1 -{\bm \theta}_2 ||  + G\frac{(1-\epsilon)}{\rho}\frac{2}{\gamma} \big|(\ell({\bm \theta}_1;z)-c_1)  -(\ell({\bm \theta}_2;z)-c_2)\big|\allowdisplaybreaks
    \\ \\ 
       \leq &\bigg(\frac{(1-\epsilon)}{\rho}+ \epsilon \bigg)\beta|| {\bm \theta}_1 -{\bm \theta}_2 || + G^2\frac{(1-\epsilon)}{\rho}\frac{2}{\gamma} ||({\bm \theta}_1,c_1) -  ({\bm \theta}_2,c_2)|| \tag{by $G-$Lipschitzness of $\ell$}
    \\ \\ = &
    \bigg(\frac{(1-\epsilon)}{\rho}(\beta + \frac{2}{\gamma} G^2) + \epsilon \beta  \bigg) ||m_1 -  m_2||,
\end{align*}
with $m_1=({\bm \theta}_1,c_1)$ and $m_2=({\bm \theta}_2,c_2)$.
\paragraph{4.} 
From Definition \ref{def:smoothed_plus_function} 
we have that for any fixed pair of $({\bm \theta},c)$:
    \begin{align*} 
        & (\ell({\bm \theta};z) - c)_+\leq s(\ell({\bm \theta};z) - c)\leq  (\ell({\bm \theta};z) - c)_+ + {\gamma} &\\   
     \Rightarrow & 
    (1-\epsilon)[c + \frac{1}{\rho}(\ell({\bm \theta};z) - c)_+] + {\epsilon} \ell({\bm \theta};z)  
     \leq (1-\epsilon)[c + \frac{1}{\rho}s(\ell({\bm \theta};z) - c)] + {\epsilon} \ell({\bm \theta};z) &\\&
     \leq (1-\epsilon)[c + \frac{1}{\rho}(\ell({\bm \theta};z) - c)_+] + {\epsilon} \ell({\bm \theta};z) + \frac{(1-\epsilon)}{\rho}{\gamma}       & \\   
     \Rightarrow &
    f({\bm \theta},c;z)\leq \Tilde{f}({\bm \theta},c;z)\leq  f({\bm \theta},c;z) + \frac{(1-\epsilon)}{\rho}{\gamma} 
    \end{align*}
\end{proof}

%% file: proofs/convergence_multirounds.tex
For the convergence analysis of Algorithm \ref{alg:fedcvar} we leverage the standard results for FedAvg presented in Theorem 1 in \cite{advancesFL2}. We apply them to our setting and provide the proof of Lemma \ref{lemma:multi_fedcvar_convergence} below.

\begin{proof}[Proof of Lemma \ref{lemma:multi_fedcvar_convergence}]

We apply the results from Theorem 1 in \cite{advancesFL2} to our setting by substituting the properties of the loss function $\ell$ with those of the non-smooth and smoothed auxiliary functions, $f$ and $\Tilde{f}$ respectively, given by Lemma \ref{lemma:properties}. 
In particular,
\begin{enumerate}
    \item $\beta$ is changed to the smoothness parameter of $\tilde{f}$, which is $\frac{(1-\epsilon)}{\rho}(\beta + \frac{2}{\gamma} G^2) + \epsilon \beta $;
    \item we use $\sigma^2$ to represent the bounded variance in the case that each client uses a batch with size $b$ instead of a single sample;
    \item $M$ is changed to $\sqrt{M^2 +B^2}$ since we optimize over $v \in \Theta\times[0,B]$.
\end{enumerate}
 
Considering also Lemma \ref{lemma:properties}, 
property 4, we finally obtain
\begin{align*}\label{finalopt_error}
&\mathop{\E}\bigg[\sum\limits_{k \in \mathcal{K}}\sum\limits_{i =1}^{n_k} \frac{ {f}(\overline{\bm  {\bm \theta}}_T,\overline{c}_T;z^k_i)}{n} \bigg] \leq
     \mathop{\E}\bigg[\sum\limits_{k \in \mathcal{K}}\sum\limits_{i =1}^{n_k} \frac{\tilde{f}(\overline{\bm  {\bm \theta}}_T,\overline{c}_T;z^k_i) }{n} \bigg]\allowdisplaybreaks
     \\ \\
     \leq&  
     \mathop{\E}\bigg[\sum\limits_{k \in \mathcal{K}}\sum\limits_{i =1}^{n_k} \frac{ \tilde{f}( {\bm \theta}^*_D,c^*_D;z^k_i)}{n} \bigg] 
     +    \frac{2\big(\frac{(1-\epsilon)}{\rho}(\beta + \frac{2}{\gamma} G^2) + \epsilon \beta  \big) ({M^2 +B^2})}{\tau T} 
         + \frac{2\sigma \sqrt{{M^2 +B^2}}}{\sqrt{|\mathcal{K}| b \tau T}}  \allowdisplaybreaks
     \\ & + 5 \bigg ( \frac{\big(\frac{(1-\epsilon)}{\rho}(\beta + \frac{2}{\gamma} G^2) + \epsilon \beta  \big) \sigma^2 ({M^2 +B^2})^2}{\tau b T^2} \bigg)^\frac{1}{3}  + 19 \bigg ( \frac{\big(\frac{(1-\epsilon)}{\rho}(\beta + \frac{2}{\gamma} G^2) + \epsilon \beta  \big)\mu^2 ({M^2 +B^2})^2}{T^2}\bigg)^\frac{1}{3}
 \\ \\   = &  \mathop{\E}\bigg[\sum\limits_{k \in \mathcal{K}}\sum\limits_{i =1}^{n_k} \frac{ \tilde{f}( {\bm \theta}^*_D,c^*_D;z^k_i)}{n} \bigg] 
 +    \frac{2\big(\frac{(1-\epsilon)}{\rho}(\beta + \frac{2}{\gamma} G^2) + \epsilon \beta  \big) ({M^2 +B^2})}{\tau T} 
      \\ &  + \frac{2\sigma \sqrt{{M^2 +B^2}}}{\sqrt{|\mathcal{K}| b \tau T}} 
      +  \bigg ( \frac{\big(\frac{(1-\epsilon)}{\rho}(\beta + \frac{2}{\gamma} G^2) + \epsilon \beta  \big)  ({M^2 +B^2})^2}{ T^2} \bigg)^\frac{1}{3} \bigg( 5(\frac{\sigma^2}{b\tau})^{\frac{1}{3}}  + 19 { \mu^{\frac{2}{3}} } \bigg)
     \\ \\ \leq &  \mathop{\E}\bigg[\sum\limits_{k \in \mathcal{K}}\sum\limits_{i =1}^{n_k} \frac{  {f}( {\bm \theta}^*_D,c^*_D;z^k_i)}{n} \bigg] 
     +    \frac{2\big(\frac{(1-\epsilon)}{\rho}(\beta + \frac{2}{\gamma} G^2) + \epsilon \beta  \big) ({M^2 +B^2})}{\tau T}   +\frac{(1-\epsilon)\gamma}{\rho}
       \\ & + \frac{2\sigma \sqrt{{M^2 +B^2}}}{\sqrt{|\mathcal{K}| b \tau T}}  
      +  \bigg ( \frac{\big(\frac{(1-\epsilon)}{\rho}(\beta + \frac{2}{\gamma} G^2) + \epsilon \beta  \big)  ({M^2 +B^2})^2}{ T^2} \bigg)^\frac{1}{3} \bigg( 5(\frac{\sigma^2}{b\tau})^{\frac{1}{3}}  + 19 { \mu^{\frac{2}{3}} } \bigg).
\end{align*}
Finally, the learning rate becomes 
 \begin{align*} 
     \eta= & \min\bigg\{  
      \frac{(M^2+B^2)^{\frac{1}{3}}}{\tau(\mu^2 \big(\frac{(1-\epsilon)}{\rho}(\beta + \frac{2}{\gamma} G^2) + \epsilon \beta  \big)  T)^{\frac{1}{3}}} ,  \frac{\sqrt{b|\mathcal{K}|}\sqrt{M^2+B^2}}{\sigma\sqrt{\tau T} },
      \bigg (\frac{b(M^2+B^2)}{\sigma^2  \tau^2 \big(\frac{(1-\epsilon)}{\rho}(\beta + \frac{2}{\gamma} G^2) + \epsilon \beta  \big)  T}\bigg )^{\frac{1}{3}}, \\ &\frac{1}{4 \big(\frac{(1-\epsilon)}{\rho}(\beta + \frac{2}{\gamma} G^2) + \epsilon \beta  \big)  }\bigg \}  
 \end{align*} 
\end{proof}

%% file: proofs/lemma_sgd_opt.tex
\begin{proof}[Proof of {Lemma \ref{lemma:excess_risk}
}]
 
For the excess risk analysis,  
we use the same substitutions as in Lemma \ref{lemma:multi_fedcvar_convergence}.  
Thus, Eq. \ref{eq:hardt_bound} for our setting becomes
\begin{equation}\label{adjusted_proposition}
\begin{array}{rl}
  &\E\bigg[  \mathop{\E}\limits_{K \in \mathcal{K}}\big[\mathop{\E}\limits_{Z|k}[\tilde{f}(\overline{{\bm \theta}}_T,\overline{c}_T;Z|k)
]\big]  
 -\frac{1}{n}\sum\limits_{k \in \mathcal{K}}\sum\limits_{i =1}^{n_k} \tilde{f}( {{\bm \theta}}^*_D, {c}^*_D;z^k_i) 
     \bigg]
\\ \\
=&
\E\bigg[
\mathop{\E}\limits_{Z}[\Tilde{f}(\overline{{\bm \theta}}_T,\overline{c}_T;Z)] \bigg] -  \E\bigg[\frac{1}{n}\sum\limits_{i=1}^n \Tilde{f}( {{\bm \theta}}^*_D, {c}^*_D;z_i)  
\bigg]
\\ \\  
\leq&  \frac{1}{2} \bigg(\frac{M^2 +B^2}{\eta T} + \frac{G_{\rho,\epsilon}^2 \eta}{ {\big ( {\sum\limits_{k \in \mathcal{K}} b_k} \big)}n}(2T+n) \bigg),
\end{array}
\end{equation}
with the learning rate being
\begin{equation}\label{learning_rate}
    \eta=\frac{\sqrt{M^2 +B^2}\sqrt{n {\big ( {\sum\limits_{k \in \mathcal{K}} b_k} \big)}}}{G_{\rho,\epsilon}\sqrt{T(n+2T)}}.
\end{equation}
Based on the remark we made above about the learning rate in Lemma \ref{lemma:hardt_convergence} being at most $\frac{2}{\beta}$, we must ensure that the respective step size in Eq. \ref{learning_rate} is at most $\frac{2{\rho}}{{(1-\epsilon)}(\beta + \frac{2}{\gamma} G^2) + \epsilon \beta\rho }$ as well. 

We know that for any $v,u>0$ we have that 
$2 \max\{v,u\} \geq v+u   \Rightarrow \frac{2}{v+u} \geq  \frac{1}{\max\{v,u\}}$. Thus, given also that $\epsilon\in (0,1]$, we obtain
\begin{align*}
        \frac{2{\rho}}{{(1-\epsilon)(\beta + \frac{2}{\gamma}G^2) +\epsilon\beta\rho }}& \geq  \frac{2{\rho}}{{\beta + \frac{2}{\gamma}G^2 +\epsilon\beta\rho }} 
         \geq \frac{\rho}{\max\{{\beta(1 +\epsilon\rho), \frac{2}{\gamma}G^2  }\}} 
         \geq \min\bigg\{\frac{{\rho}}{\beta(1 +\epsilon\rho)}, \frac{\rho\gamma}{2G^2} \bigg\}.
\end{align*}

Thus, it is sufficient  to ensure that (i) $\eta\leq\frac{{\rho}}{\beta (1+\epsilon\rho)}$,
and (ii) $\eta\leq\frac{\rho\gamma}{2G^2}$. 

We can satisfy the first case, $\eta\leq\frac{{\rho}}{\beta (1+\epsilon\rho)}$, by the choice of the rounds number $T$, that is
\begin{equation*}
    \begin{array}{rl}
         \eta=\sqrt{M^2 +B^2}\frac{\sqrt{n{\big ( {\sum\limits_{k \in \mathcal{K}} b_k} \big)}}}{G_{\rho,\epsilon}\sqrt{T(n+2T)}} \leq\frac{{\rho}}{\beta (1+\epsilon\rho)}            \\ \\
          \Longrightarrow n{\big ( {\sum\limits_{k \in \mathcal{K}} b_k} \big)}\big (M^2 +B^2\big ) \bigg (\frac{\beta (1+\epsilon\rho) }{\rho G_{\rho,\epsilon}}\bigg )^2\leq {T(n+2T)}.
    \end{array}
\end{equation*}

For the case $\eta\leq\frac{\rho\gamma}{2G^2}$, since $\rho \in (0,1)$,  $\epsilon\in (0,1]$ and
\begin{equation*}
    \begin{array}{rl}
         G_{\rho,\epsilon}&=\max \bigg\{ \sqrt{G^2  \epsilon^2 + (1-\epsilon)^2}, \sqrt{\frac{G^2(1-\epsilon +\epsilon \rho)^2 + (1-\epsilon)^2({\rho}-1)^2 }{\rho^2}} \bigg\}
         \geq \max \bigg\{  \epsilon G, \frac{ (1-\epsilon +\epsilon \rho) G}{\rho} \bigg\} \geq \frac{ (1-\epsilon +\epsilon \rho) G}{\rho} 
    \end{array}
\end{equation*}
we have that
\begin{equation}
\frac{\rho\gamma}{2G^2}  \geq \frac{\rho^2\gamma}{2G^2}  \geq \frac{(1-\epsilon +\epsilon \rho)^2\gamma}{2G_{\rho,\epsilon}^2}.
\end{equation}

Thus, by setting $\gamma= \frac{2 G_{\rho,\epsilon}^2}{(1-\epsilon +\epsilon \rho)^2}\eta$ we can satisfy condition (ii).

Finally, we yield the proposed bound by using the learning rate in Eq.  \ref{learning_rate} 
and the results in Lemma  \ref{lemma:properties} 
for which Eq. \ref{adjusted_proposition} becomes
\begin{equation}
\begin{array}{ll}
     \E\big[
     \mathop{\E}\limits_{Z}[{f}(\overline{{\bm \theta}}_T,\overline{c}_T;Z)]
     \big]
      \leq 
     \E\big[
     \mathop{\E}\limits_{Z}[\Tilde{f}(\overline{{\bm \theta}}_T,\overline{c}_T;Z)]
     \big]  \\ \\
         \leq  \E\big[\frac{1}{n}\sum\limits_{i=1}^n \Tilde{f}({{\bm \theta}}^*_D,{c}^*_D;z_i)  \big] +   G_{\rho,\epsilon}\sqrt{ \frac{(M^2 +B^2)(\frac{2}{n}+\frac{1}{T})} {\big ( {\sum _{k \in \mathcal{K}} b_k} \big)}} 
     \\ \\\leq  \E\big[\frac{1}{n}\sum\limits_{i=1}^n {f}( {{\bm \theta}}^*_D,{c}^*_D;z_i) \big] +   G_{\rho,\epsilon}\sqrt{ \frac{(M^2 +B^2)(\frac{2}{n}+\frac{1}{T})} {\big ( {\sum\limits_{k \in \mathcal{K}} b_k} \big)}}  +\frac{(1-\epsilon)\gamma}{\rho}
\end{array}
\end{equation}
\end{proof}

%% file: appendix/experimental_details.tex
\section{Experimental Details}\label{appdx:experiments}
 
\subsection{Experimental Setup} 
We preprocess ACS Employment as described in \cite{ding2021retiring} and eICU akin to \cite{Pollard2018}. eICU dataset requires credentialed access and the procedure for requesting access is described on the dataset's website \url{https://eicu-crd.mit.edu/gettingstarted/access/}. 
{For MNIST we used a setting akin to  FashionMNIST splits in \cite{DRFA,AFL}. For ACS Employment we consider two settings: (a) a setting with 51 clients, where each client represents a different geo-location and (b) the hardest data partitioning using a sensitive group, proposed in \cite{papadaki}, where the data is split to $3$ clients based on the race classes: \{Black, White, Others\} to make the comparison in Appendix \ref{appdx:Additional_res} fairer. 
For Celeb-A, we randomly assign the partitions across two clients using three different seeds.}
For the ACS Employment and MNIST datasets, we use a MLP with a single hidden layer with $512$ neurons. For eICU we use logistic regression and for Celeb-A we use a ResNet-18. We select cross entropy to be the loss function in every training scenario. 

We report results such as the worst performing group and average utility that directly relate to the notion of minimax fairness. We also present the risk disparity between best and worst groups, conditioned on group size $\rho$. 
We train FedSRCVaR using local batch size $b_k=\{32,64,128\}$ and $\epsilon \approx 0.0$ is set as $\epsilon=0.01$, except for ACS Employment that we pick the best solution from $\epsilon=\{0.001,0.005,0.01,0.05\}$. We use the same batch size options for AFL. BPF is trained using $\epsilon=\{ 0.001,0.005,0.01,0.05\}$. FedAvg is trained using batches of sample size $128$ and local epochs $E=\{3,8,15\}$. We train all approaches using learning rates $\eta=\{0.01,0.001,0.0001,0.00001\}$, adversary/threshold learning rate $\eta_{adv}\break=0.001$ (where relevant). We pick the combination with the best solution for each case. For the proposed approach, FedSRCVaR, we report the results for group size $\rho =\{0.1,0.2,0.3,0.4,0.5,0.6,0.7,0.8,\break0.9\}$ and trade-off parameter $\epsilon=\{0.001,0.005,0.01,0.05,0.1,0.2,0.3,  0.4,0.5,0.6,0.7,0.8,0.9,1.0\}$ unless stated otherwise. 

In the figures we present the mean performance over three runs and in separate splits. The splits are generated using 3-fold cross-validation. If a fixed test set is provided by the authors of the dataset we use that for testing. For the experiments, we use soft-ReLU as a smoothed plus function with $\gamma=0.05$. 

\subsection{Implementation \& Training Devices} The experiments were implemented in Python using PyTorch. We produce results for BPF using the original code available at \url{github.com/natalialmg/BlindParetoFairness}. The experiments were realised using $4 \times$ NVIDIA Tesla V100 GPUs.

\subsection{Additional Empirical Results}\label{appdx:Additional_res}

We examine the cost in (minimax) group fairness when considering sensitive groups that are (potentially incorrectly) anticipated in the testing phase.
This scenario reflects realistic situations where we might lack access to information regarding the future evolution of demographics, even if we are aware of the sensitive demographic groups during training time. Hence, we demonstrate how a model trained on known demographic groups will perform when faced with the most challenging, worst-case scenario within these groups.
\begin{table}[H]
\centering
\caption{Cross Entropy risks comparison of minimax Pareto federated group fairness with real demographics (FedMinMax), unknown demographics (FedSRCVaR, ours) and baseline (FedAvg) on ACS Employment dataset. Results are averaged over 3 runs. }
\label{tab:known_groups}

\resizebox{0.6\columnwidth}{!}{
\begin{tabular}{lccccccc }
\toprule
 \multicolumn{2}{c}{\textbf{Group}} &\textbf{FedMinMax} &\multicolumn{2}{c}{ \textbf{FedSRCVaR} ($\epsilon=0.05$)}& \textbf{FedAvg} \\
\multicolumn{2}{c}{} & & $\rho= 0.1$ & $\rho= 0.3$& (Baseline)
\\  
\midrule
 Worst & $\rho= 0.1$ & 1.593$\pm$0.23 &0.713$\pm$0.07 & 0.724$\pm$0.06& 1.768$\pm$0.04 \\
group& $\rho= 0.3$ &1.176$\pm$0.08 &0.698$\pm$0.02& 0.695$\pm$0.05& 1.037$\pm$0.09 \\
\bottomrule
\end{tabular}
}
\end{table}
We compare against the optimal minimax group fair solution for known sensitive groups, generated using FedMinMax \cite{papadaki}, and present the results in Table \ref{tab:known_groups}. We leverage ACS Employment dataset and allocate data on $3$ clients based on the races \{Black, White, Others\},
as in \cite{papadaki}. 
We observe that FedMinMax has significantly poor performance on the worst groups of the selected $\rho$ sizes. On the other hand, FedSRCVaR outperforms FedMinMax in the worst group generated by all samples. 
We emphasize that our approach does not utilize existing groups. FedSRCVaR optimizes for the worst-case group that can be formulated from the training data, ensuring that \textit{no (real) group will perform worse than that}. However, the worst-case group may not be easily described as one of the commonly defined demographic groups; this subset of bad-performing samples may be distributed across the predefined group categories. On the other hand, FedAVG focuses solely on (average) utility, which means it will perform poorly on the worst possible subgroup.

These results indicate that the price of optimizing for unknown demographics is lower than the cost of optimizing for wrong demographics, given by the groups-agnostic approach. Hence, we argue that FedSRCVaR is not only beneficial for settings where the sensitive groups are completely unknown, but also it is preferable when the known sensitive groups could potentially change in the future; or in the general case that we are not completely certain that the sensitive demographics remain the same during training and testing time. 

For reference, we also provide the proportion of the various demographics in the predicted worst group that is generated by our approach for $ \epsilon\approx0$ and low $\rho$ values; and of the actual group populations in Table \ref{tab:groups_composition_fedsrcvar}. We notice that the composition of the worst group is very close to the actual size of the sensitive groups, especially as $\rho$ grows. 

\begin{table}[H]
\centering 
\caption{Worst group formation and actual group size on the test split on ACS Employment dataset. FedSRCVaR is evaluated for $\epsilon=0.05$ and $\rho=\{0.1,0.3\}$. We denote as \{U, E\} the labels \{Unemployed, Employed\}. }
\label{tab:groups_composition_fedsrcvar}
\resizebox{0.45\textwidth}{!}{
\begin{tabular}{llcccccc}
\toprule
& &\multicolumn{2}{c}{\textbf{White} (\%)} &\multicolumn{2}{c}{\textbf{Black} (\%)} &\multicolumn{2}{c}{\textbf{Other} (\%)}
\\
\cmidrule(lr){3-4}\cmidrule(lr){5-6}\cmidrule(lr){7-8}
\textbf{$\bm \epsilon$ } & \textbf{$\bm  \rho$ } &  {U} &  {E} &   {U} &  {E} &  {U} &  {E} \\
\midrule
0.05&0.1& 48.6 & 14.9 &4.3 & 1.2& 24.1 & 6.9  \\
&0.3&  30.2& 30.7  & 2.6 &2.& 17.8 & 16.7   \\
\midrule
\multicolumn{2}{c}{Actual size}     &33.4 & 27.9  & 2.9 &  1.9  & 18.3& 15.6 \\
\bottomrule
\end{tabular}}
\end{table}